\documentclass[twoside]{article}

\usepackage[accepted]{aistats2024}
\usepackage{natbib}
\usepackage[utf8]{inputenc} 
\usepackage[T1]{fontenc}    
\usepackage{xcolor}\definecolor{darkblue}{rgb}{0.0,0,0.75} 
\usepackage[hidelinks=false,colorlinks=true, citecolor=darkblue, urlcolor=darkblue,linkcolor=darkblue]{hyperref}       
\usepackage{url}            
\usepackage{booktabs}       
\usepackage{amsfonts}       
\usepackage{nicefrac}       
\usepackage{microtype}      
\usepackage{xcolor}         
\usepackage{subfigure}
\usepackage{amsmath}
\usepackage{amssymb}
\usepackage{mathtools}
\usepackage{amsthm}
\usepackage{thmtools}
\usepackage{thm-restate}
\usepackage{multirow}
\usepackage{algorithm}
\usepackage{comment}
\usepackage[noend]{algorithmic}
\setcitestyle{round}
\newtheorem{theorem}{Theorem}
\theoremstyle{definition}

\newtheorem{proposition}[theorem]{Proposition}
\newcommand{\defemph}[1]{\textbf{#1}}
\DeclareMathOperator*{\argmax}{arg\,max}
\DeclareMathOperator*{\argmin}{arg\,min}

\newcommand{\spanof}{\mathrm{span}}
\newcommand{\range}{\mathrm{range}}

\newcommand{\from}{\!:\,}
\newcommand{\with}{\!:\,}
\newcommand{\given}{\, | \,}
\newcommand{\inner}[2]{#1^T #2}
\newcommand{\R}{\mathbb{R}}

\newcommand{\cB}{\mathcal{B}}
\newcommand{\cQ}{\mathcal{Q}}

\newcommand{\matO}{\mathbf{O}}

\newcommand{\Q}{\mathbf{Q}}
\newcommand{\X}{\mathbf{X}}

\newcommand{\e}{\mathbf{e}}
\newcommand{\f}{\mathbf{f}}
\newcommand{\g}{\mathbf{g}}
\newcommand{\h}{\mathbf{h}}
\newcommand{\q}{\mathbf{q}}
\newcommand{\x}{\mathbf{x}}
\newcommand{\vecv}{\mathbf{v}}

\newcommand{\y}{\mathbf{y}}
\newcommand{\veco}{\mathbf{o}}
\newcommand{\0}{\mathbf{0}}

\newcommand\norm[1]{\lVert#1\rVert}
\newcommand{\expect}{\mathrm{\bf E}}

\newcommand{\alphavec}{\boldsymbol{\alpha}}
\newcommand{\weight}{\beta}
\newcommand{\weights}{\boldsymbol{\weight}}
\newcommand{\activation}{\mu^{-1}}

\newcommand{\gb}{\mathrm{gb}}
\newcommand{\gs}{\mathrm{gs}}
\newcommand{\ogb}{\mathrm{ogb}}
\newcommand{\xgb}{\mathrm{xgb}}

\newcommand{\obj}{\mathrm{obj}}
\newcommand{\bnd}{\mathrm{bnd}}

\newcommand{\ogbobj}{\obj_\ogb}
\newcommand{\gbobj}{\obj_\gb}
\newcommand{\gsobj}{\obj_\gs}
\newcommand{\xgbobj}{\obj_\xgb}

\title{Orthogonal Gradient Boosting for Simpler\\Additive Rule Ensembles}

\author{%
  Fan Yang 
  Department of Computer Science\\
  Monash University\\
   \\
  \texttt{fan.yang1@monash.edu} \\
}

\begin{document}

\twocolumn[
\aistatstitle{Orthogonal Gradient Boosting for Simpler Additive Rule Ensembles}

\aistatsauthor{ Fan Yang \And Pierre Le Bodic \And  Michael Kamp \And Mario Boley }

\aistatsaddress{ Monash University \And  Monash University \And IKIM, University Hospital Essen \And Monash University } ]

\begin{abstract}
Gradient boosting of prediction rules is an efficient approach to learn potentially interpretable yet accurate probabilistic models. However, actual interpretability requires to limit the number and size of the generated rules, and existing boosting variants are not designed for this purpose. Though corrective boosting refits all rule weights in each iteration to minimise prediction risk, the included rule conditions tend to be sub-optimal, because  commonly used objective functions fail to anticipate this refitting.
Here, we address this issue by a new objective function that measures the angle between the risk gradient vector and the projection of the condition output vector onto the orthogonal complement of the already selected conditions.  This approach correctly approximate the ideal update of adding the risk gradient itself to the model and favours the inclusion of more general and thus shorter rules. As we demonstrate using a wide range of prediction tasks, this significantly improves the comprehensibility/accuracy trade-off of the fitted ensemble. Additionally, we show how objective values for related rule conditions can be computed incrementally to avoid any substantial computational overhead of the new method.
\end{abstract}


\section{Introduction}\label{intro}
\begin{figure}[tb]
\vskip 0.1in
\begin{center}
\begin{subfigure}
\centering
\includegraphics[width=\columnwidth]{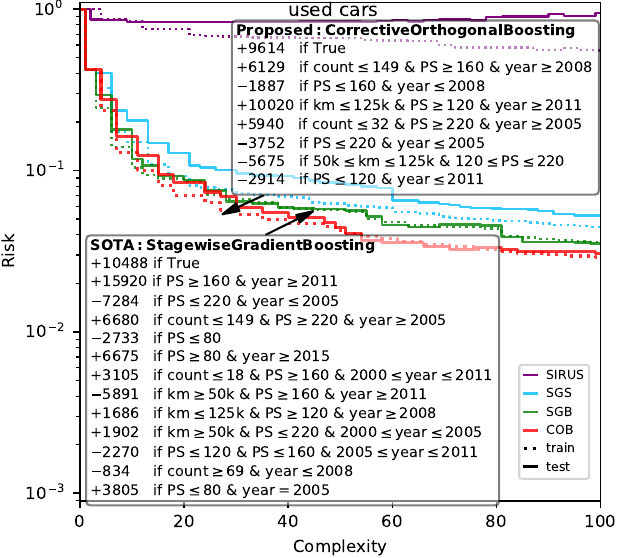}
\vspace{-0.5cm}
\caption{Risk/complexity trade-offs achieved by different rule ensemble learners 
for  \texttt{used\_cars}. 
Proposed approach (red) reaches risk 0.135 with ensemble of complexity 27, measured as number of rules plus their total length, as opposed to 46 for best alternative.
}
\label{fig:compare_tic}
\end{subfigure}
\label{fig:risk_complexity_tradeoff}
\end{center}
\end{figure}
\defemph{Additive rule ensembles} are probabilistic models that describe the mean of a target variable $Y$ conditional on an input variable $X$ as $\expect[Y \given X=\x] = \activation(f(\x))$ where $\activation \from \R \to \R$ is an inverse link or \defemph{activation function} and $f\from \R^d \to \R$ is an affine linear combination
of $k$ Boolean \defemph{query functions}  that correspond to conjunctions of threshold functions on individual input coordinates. 
That is, 
\begin{equation}
f(\x) = \weight_0 + \sum_{i=1}^{k}{\weight_i q_i(\x)}
\label{eq:add_rule_ensemble}
\end{equation}
where each $q_i\from \R^d \to \{0, 1\}$ is a product $q_i(\x)=p_{i,1}(\x)p_{i,2}(\x)\dots p_{i,c_i}(\x)$ of \defemph{propositions} $p_{i,j}(\x)=\delta(sx_l \geq t)$ with $s \in \{\pm 1\}$ and $t \in \R$.
Each term in~\eqref{eq:add_rule_ensemble} can be interpreted as an IF-THEN rule where the binary queries $q_i$ define the rule antecedents (conditions), and the \defemph{weights} $\weights=(\weight_1,\dots,\weight_k) \in \R^k$ define the rule consequents, i.e., the output of rule $i$ for input $\x\in \R^d$ is $\weight_i$ if $\x$ satisfies the antecedent, i.e., $q_i(\x)=1$ (and $0$ otherwise).
The offset weight $\weight_0$ conceptually describes the output of a \defemph{background rule} with a trivial condition that is satisfied for all data points.

These models, also called ``rule sets'', are of interest, because they combine an interpretable syntax with high modelling flexibility.
Thus, they are a useful alternative to employing opaque models with post-hoc explanations~\citep[e.g.,][]{strumbelj2010efficient, ribeiro2016should}, which can be misleading~\citep{rudin2019stop, kumar2021shapley}. 
Indeed, like simple generalized linear models~\citep[GLMs,][]{mccullagh2019generalized} and generalized additive models \citep[GAMs,][]{Hastie1990, lou2013accurate}, additive rule ensembles are ``modular'' and ``simulatable''~\citep{murdoch2019interpretable}, i.e., they allow the interpretation of one term (rule) at a time, and the output of each term can be computed by a human interpreter ``in her head''.
However, in contrast to those simpler models, rule ensembles can also model interaction effects of an arbitrary number of input coordinates.
In fact, they can be regarded as a generalization of tree and forest models by representing each tree leaf through a rule.
Though, given the motivation of interpretability, we are interested in finding much smaller rule ensembles than those typically defined by a forest. 
In other words, the goal in rule ensemble learning really is to optimize the trade-off between accuracy and complexity (see Fig.~\ref{fig:risk_complexity_tradeoff}).


Algorithms for learning additive rule ensembles range from computationally inexpensive generate-and-select approaches~\citep{friedman2008predictive, lakkaraju2016interpretable, benard2021interpretable}, over more expensive minimum-description length and Bayesian approaches~\citep{wang2017bayesian}, to expensive full-fledged discrete optimization methods~\citep{dash2018boolean, wei2019generalized}.
Within this range of options, methods based on \emph{gradient boosting}~\citep{friedman2001greedy} are of special interest because of their good accuracy relative to their  cost and their flexibility to adapt to various response variable types and loss functions.
These methods identify one rule condition at a time by optimizing an objective function that aims to approximate gradient descent of the empirical risk in terms of the ensemble's prediction vector.
Current boosting adaptions to rule learning~\citep{cohen1999simple, dembczynski2010ender, boley2021better} are, however, based on design choices that compromise the risk/complexity trade-off of the fitted rule ensembles.
They use stagewise weight updates where rules are not revised after they are added to the ensemble,
although re-optimizing all weights in every iteration in a (totally) \emph{corrective update}~\citep{kivinen1999boosting, shalev2010trading} can achieve a smaller risk for the same ensemble complexity.
Interestingly, fixing this issue is not as straightforward as simply switching to a corrective update, because current boosting objective functions
do not anticipate weight corrections, and, as we show here, this can lead to highly sub-optimal choices and in particular fails to guarantee finding the rule condition that best approximates the inclusion of the risk gradient itself to the multi-dimensional weight search.

Here we investigate a new objective function that provides this guarantee and leads to consistent gains in the risk reduction per added rule. This function is based on considering only the part of a rule body orthogonal to the already selected rules, which takes into account that the predictions for previously covered training examples can be adjusted during weight correction. In addition to providing this refined objective function, we also derive a corresponding  algorithm for incrementally evaluating sequences of related rule conditions, which is crucial for efficiently optimizing the objective function in each boosting round, whether through greedy or branch-and-bound search. As we demonstrate on a wide range of datasets, the resulting rule boosting algorithm significantly outperforms the previous boosting variants in terms of risk/complexity trade-off, which can be attributed to a better risk reduction per rule as well as an affinity to select more general rules. At the same time, the computational cost remains comparable to previous objective functions.
We present these main technical contributions in Sec.~\ref{fcogb} and their empirical evaluation in Sec.~\ref{experiments} and provide a concluding discussion in Sec.~\ref{conclustions}. To start, we briefly recall gradient boosting for rule ensembles in Sec.~\ref{grb}.

\section{Rule Boosting}\label{grb}


We are concerned with the trade-off of two properties of an additive rule ensemble $f$: the \defemph{regularized empirical risk} and its \defemph{complexity}. The first is defined as
$
    R_\lambda(f) = \sum_{i=1}^{n}{l(f(\x_i), y_i)}/n + \lambda \|\weights\|^2/n
$,
with some positive loss function $l(f(\x), y)$ averaged over a training set, $\{(\x_1, y_1),\dots, (\x_n, y_n)\}$ sampled with respect to the joint distribution of $X$ and $Y$, and estimates the \defemph{prediction risk} $\expect[l(f(X), Y)]$ for new random data $(X, Y)$.
The second is defined as $C(f) = k + \sum_{i=1}^k c_i$ and approximates the cognitive effort required to parse all rule consequents and antecedents.
Here we consider loss functions that can be derived as negative log likelihood (or rather, deviance function) when interpreting the rule ensemble output as natural parameter of an exponential family model of the target variable, which  guarantees that the loss function is strictly convex and twice differentiable.
Specifically, we consider the cases of $Y \given X$ being normally distributed and $\activation(a)=a$ resulting in the \defemph{squared loss} $l_\mathrm{sqr}(f(x_i), y_i) = (f(x_i)-y_i)^2$, Bernoulli distributed and $\activation(a)=1/(1+\exp(-a))$ resulting in the \defemph{logistic loss} $l_\mathrm{log}(f(x_i), y_i)=\log(1+\exp(-y_if(x_i)))$, and Poission distributed and $\activation(a)=\exp(a)$ resulting in the \defemph{Poisson loss} $l_\mathrm{poi}(f(x_i), y_i)=y_i\log y_i-y_if(x_i)-y_i+\exp(f(x_i))$.

\paragraph{Gradient boosting}
Gradient boosting methods are iterative fitting schemes for additive models that, in our context, produce a sequence of rule ensembles $f^{(0)}, f^{(1)}, \dots, f^{(k)}$ starting with an empty  $f^{(0)}=\beta_0$ where $\beta_0 = \argmin \{R_0(\mathbf{1}\beta)\}$, and for $1 \leq t \leq k$,
\begin{align*}
    f^{(t)}(\x) &= \beta_0 + \beta_{t, 1}q_1(\x) +\dots + \beta_{t, t}q_t(\x)\\
    q_t &= \argmax \{\obj(\q; f^{(t-1)})\with {q \in \cQ}\},
\end{align*}
where query $q_t$ is chosen to maximize some objective $\obj$ that is defined on the query output vector $\q=(q(x_1),\dots, q(x_n))$ and that additionally depends on the previous model, in particular its output vector $\f_{t-1}=(f^{(t-1)}(\x_1), \dots, f^{(t-1)}(\x_n))$.
Specifically, the original gradient boosting scheme~\citep{friedman2001greedy} performs an approximate gradient descent in the model output space by choosing 
\begin{align}
    \gbobj(\q) &= |\inner{\g}{\q}|/\|\q\| \enspace , \label{eq:gb_obj}\\
    \weights_t &= [\weights_{t-1}; \argmin_{\weight \in \R}R_\lambda(\f_{t-1}+\weight q_t)],\label{eq:gb_weight_update}
\end{align}
where $\g$ is the \defemph{gradient vector} of the unregularized empirical risk function with components $g_i=\partial l(f(\x_i), y_i) / \partial f(\x_i)$.
We refer to \eqref{eq:gb_obj} as \defemph{gradient boosting objective} and to \eqref{eq:gb_weight_update} as \defemph{stagewise weight update}, because it fits the final prediction function in ``stages'' corresponding to each term.

One popular variant of this scheme chooses queries that correspond to a direction in output space with maximum rate of change in risk according to its first order approximation~\citep{mason1999functional}. This results in the \defemph{gradient sum objective}
\begin{equation}
    \gsobj(\q) = |\inner{\g}{\q}| 
\end{equation}
which is guaranteed~\citep[Thm.~1]{dembczynski2010ender} to select rules as least as general as the gradient boosting objective in terms of the number of selected data points $\|\q\|_1$.
This increased generality, however, can come at the expense of a reduced risk reduction when choosing $\weights$, because the output correction of data points with large gradient elements has to be toned down to avoid over-correction of other selected data points with small gradient elements.
Another variant~\citep{chen2016xgboost, boley2021better} minimizes the second order approximation to $R_\lambda$ via the \defemph{extreme boosting objective} and corresponding closed-form weight updates
\begin{align}
    \xgbobj(\q) &= |\inner{\g}{\q}|/\sqrt{\inner{\h}{\q}+\lambda}\\
    \weights_t &= [\weights_{t-1}; -\inner{\q}{\g}/(\inner{\q}{\h} + \lambda)],\label{eq:xgb_weights}
\end{align}
where $\h=\text{diag}(\nabla^2_{f(\textbf{x})}{R(f)})$ is the diagonal vector of the unregularized risk Hessian  again with respect to the output vector $\f$.
Note that this approach is well-defined for our loss functions derived from exponential family response models, which guarantee defined and positive $\h$.
In particular for the squared loss, it is equivalent to standard gradient boosting, because the second order approximation is exact for $l_\mathrm{sqr}$ and $\h$ is constant.

While the closed-form solution of the weight-updates~\eqref{eq:xgb_weights} reduces the update cost by some amount, it can be highly sub-optimal whenever the second order approximation is loose as, e.g., is the case with the Poisson loss.
Especially if one aims for small interpretable rule ensembles and correspondingly wants to minimize the risk  for each ensemble size, it appears sensible to choose $\weights$ that minimize the actual empirical risk.
The most consequent realization of this idea is given by \defemph{corrective boosting}\footnote{In the literature, the update~\eqref{eq:weight_correction} is often referred to as ``totally corrective'' or ``fully corrective''  whenever the unqualified term  ``corrective boosting'' is already used to refer to the standard stagewise update~\eqref{eq:gb_weight_update} .}~\citep{kivinen1999boosting, shalev2010trading} where the component-wise weight updates~\eqref{eq:gb_weight_update} are replaced by a full joint re-optimization of the weights of all selected queries, i.e.,
\begin{equation}
    \weights_t = \argmin \{R_\lambda(\Q_t \weights)\with \weights \in \R^t\}
    \label{eq:weight_correction} \enspace ,
\end{equation}
where $\Q_t=[\q_1, \dots, \q_t]$ is the $n\times t$ \defemph{query matrix} with the output vectors of all selected queries as columns.
Given that our loss function $l$ and therefore the empirical risk $R_\lambda$ are convex, the additional computational cost for solving~\eqref{eq:weight_correction} instead of~\eqref{eq:gb_weight_update} is negligible relative to the cost of query optimization, especially when targeting small ensembles sizes $k$.

\paragraph{Single rule optimization}
While the rule optimization literature can be broadly divided into  branch-and-bound and greedy (or more generally beam) search, these approaches are actually closely related: 
both searches start with a trivial query, $q(\x)=1$ for all $\x \in \R^d$, enqueued on a priority queue of candidates and then, until the queue is exhausted, iteratively (i) dequeue the top candidate $q$, (ii) update the current best solution $q^*$ if $\obj(q) > \obj(q^*)$, and (iii) enqueue new candidates.
For beam search with beam width parameter $k$, these candidates are $k$ best augmentations $q'=qp$ in terms of $\obj$. In contrast, branch-and-bound enqueues all augmentations $q'$ that pass a non-redundancy check\footnote{In efficient implementations this non-redundancy check ensures that each query output vector $\q$ is at most enqueued once through exploitation of a closure system formed by the queries (see \citet{boley2021better} and references therein).} if $q$ passes a bound check $\bnd(q) \geq \obj(q^*)$.
If $\bnd(q)$ is an upper bound to all possible subsequent augmentations to $q$, it is called \defemph{admissible}, and finding the optimal query is guaranteed.

Crucially, both search variants rely on efficient incremental computation of the objective function for sequences of candidate queries to avoid an overall quadratic computational complexity in the number of data points $n$. Specifically, they reduce the candidate generation and filtering to the following \textbf{prefix optimization problem}: for a given ordered sub-selection of $l$ data points $\sigma: \{1, \dots, l\} \to \{1, \dots, n\}$, find 
\begin{equation}
    \argmax \{\obj(\e_{\sigma(1)}+\dots+ \e_{\sigma(i)}) \with 1 \leq i \leq l)\}
    \label{eq:prefix_opt}
\end{equation}
in time $O(l)$, where we denote by $\e_i$ the $i$-th standard unit vector and observe that the objective functions are functions of the query output vectors on the training data, i.e., $\obj(q)=\obj(\q)$. 
For beam search, this problem is solved once per input variable $j$ and $s \in \{\pm 1\}$, setting $\sigma$ to order according to $sx_j$, to find the best augmentation $q\delta(sx_l \geq t)$ with $t \in \{x_{i, j}\}$.
For branch-and-bound, the problem is solved twice per dequeued candidate $q$ to compute the bounding function.
Here, $\sigma$ describes the sub-selection of the $l=\boldsymbol{1}^T\q$ data points selected by $q$, and the order varies based on the objective function.
For $\gbobj$, the order is with respect to $sg_i$, and for $\xgbobj$, the order is with respect to the ratio $sg_i/h_i$, in both cases leading to admissible bounding functions \citep{boley2021better}.

\section{Orthogonal Gradient Boosting}\label{fcogb}
\begin{figure*}[t!]
\centering
\includegraphics[width=0.67\columnwidth]{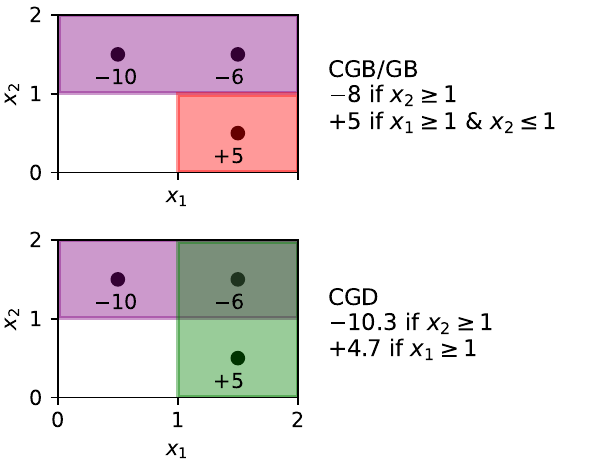}
\includegraphics[width=0.63\columnwidth]{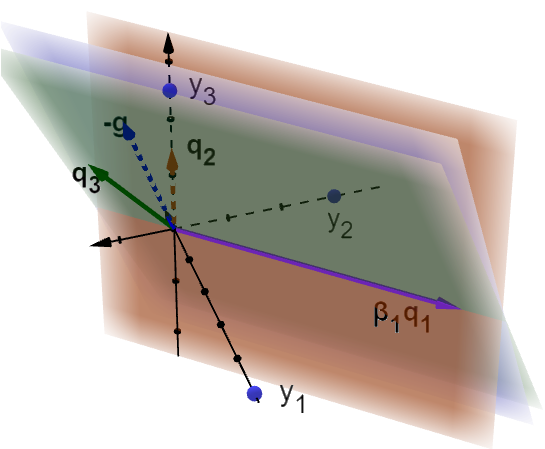}
\includegraphics[width=0.63\columnwidth]{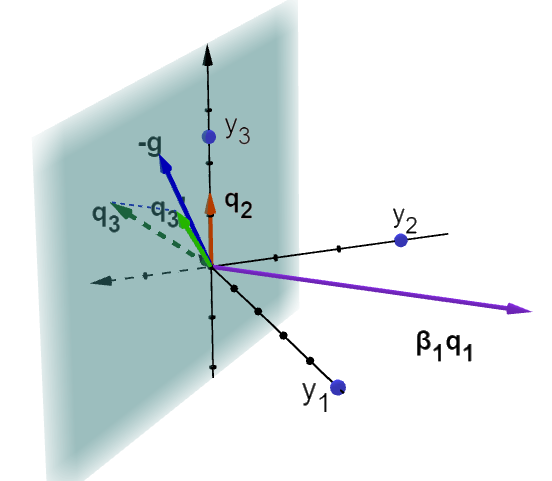} 
\vskip -0.1in
\caption{
A regression example with three data points with target values $\y=(-10, -6, 5)$ and three queries with outputs $\q_1=(1, 1, 0)$, i.e., $q_1$ selects the first two data points, $\q_2=(0, 0, 1)$, and $\q_3=(0, 1, 1)$. Gradient boosting  selects $q_1$ with weight $\beta_1=-8$ as first rule, resulting in negative gradient  $-\g=(-2, 2, 5)$.
\emph{Left:} The input space and the rule ensembles generated by CGB/GB and CGD. The CGB method generates output $(-8, -8, 5)$, and CGD generates output ${-10.3, -5.6, 4.7}$.
\emph{Middle:} Approximations to target subspace (blue) spanned by $\q_1$ and $-\g$. The subspace (green) spanned by $\q_3$ and $\q_1$ is a better approximation than the subspace (orange) selected by standard gradient boosting (spanned by $\q_2$ and $\q_1$). 
\emph{Right:} After projection onto orthogonal complement of already selected query, angle between $\q_3$ and $-\g$ is smaller than that between $\q_2$ and $-\g$ and is thus successfully selected by orthogonal gradient boosting objective.}
\label{fig:subspace}
\end{figure*}

Using the corrective weights updates~\eqref{eq:weight_correction} turns boosting into a form of forward variable selection for linear models.
However, in contrast to conventional variable selection where all variables are given explicitly, we still have to identify a good query $q_t$ in each boosting iteration, 
and it turns out that finding the appropriate query is more complicated as in the case of single weight optimization~\eqref{eq:gb_weight_update}.

\subsection{Best Geometric Approximations}
\newcommand{\bv}{\mathbf{v}}
To be more precise, assume we are adding to the ensemble a basis function with output vector $
\bv$ in boosting round $t$, and let us denote by $A^\mathrm{GD}_\bv = \f_{t-1}+\spanof\{\bv\}$ the affine weight optimization subspace of the single weight update~\eqref{eq:gb_weight_update} and by $A^\mathrm{CGD}_\bv =\range[\Q_{t-1}; \bv]$ the weight optimization subspace of the corrective weight update~\eqref{eq:weight_correction}. 
Moreover, for some family of weight optimization spaces $A_\bv$ containing a target vector $\f \in A_\bv$, we define the \defemph{best geometric approximation} with respect to the query language $\cQ$ as 
\begin{equation*}
    \tilde{\f}_\cQ=\argmin \{\|\f - \f'\| \with \f' \in A_\q, q \in \cQ\} \enspace .
\end{equation*}

In the case of the single weight update, it is easy to show that for any $\f=\f_{t-1}+\alpha\bv \in A^\mathrm{GD}_\bv$,the squared projection error for a given $q \in \cQ$ is:
\begin{equation}
    \min_{\beta \in \R} \|\alpha \bv - \beta \q \|^2 = \alpha^2\left(\|\bv\|^2 - \frac{(\inner{\q}{\bv})^2}{\|\q\|^2}\right) \enspace .
    \label{eq:1d_proj_error}
\end{equation}
Hence, if we choose the target output vector $\bv$ to be the one resulting from the ideal gradient descent update at round $t$, i.e., 
\begin{equation*}
    \f^\mathrm{GD}=\argmin \{R_\lambda(\f') \with \f' \in A^\mathrm{GD}_\g\},
\end{equation*}
then its best geometric approximation $\tilde{\f}^\mathrm{GD}_\cQ$ is an element of the affine subspace $A^\mathrm{GD}_\q$ where $\q$ denotes the output vector of a query that minimizes the standard gradient boosting objective $\gbobj$. Consequently, as the final output vector $\f_\mathrm{gb}$ is found via line search in this affine subspace, its risk is no greater than that of the best geometric approximation, i.e., $R_\lambda(\f_\mathrm{gb}) \leq R_\lambda(\tilde{\f}^\mathrm{GD}_\cQ)$.

However, when considering the \emph{corrective} gradient descent update and the target vector 
\begin{equation}
    \f^\mathrm{CGD}=\argmin \{R_\lambda(\f') \with \f' \in A^\mathrm{CGD}_\g\} \enspace.
\end{equation}
corresponding to adding the gradient itself to the optimization space, all previously discussed objective functions fail to identify the best geometric approximation in general.
For the standard gradient boosting objective, this is demonstrated in the example in Fig.~\ref{fig:subspace}.
Here, $\gbobj$ identifies in round $2$ a query resulting with or without weight correction in an output vector $\f^\mathrm{CGB}_\mathrm{gb}=\f^\mathrm{GB}_\mathrm{gb}$ with (unregularized) risk $24/9$, whereas the best geometric approximation to the target vector $\tilde{\f}_\cQ^\mathrm{CGD}$ has a risk of only $1/9$.


\subsection{Objective for Corrective Boosting}

The intuitive reason for the gradient boosting objective failing to identify the correct query in Fig.~\ref{fig:subspace} is that selecting $x_2$ in addition to $x_3$ is not beneficial for the overall risk reduction \emph{if} we are only allowed to set the weight for the newly selected query. 
This is because then this weight has to be a compromise between the two different magnitudes of correction required for $x_2$, which only needs a small positive correction, and $x_3$, which needs a large positive correction.
If we, however, are allowed to change the weight of the previously selected query this consideration changes, because we can now balance an over-correction for $x_2$ by adjusting the weight of the first rule.
While on first glance it seems unclear how much of such re-balancing can be applied without harming the overall risk,
it turns out that this is captured by a simple criterion based on the norm of the part of the newly selected query that is orthogonal to the already selected ones (see SI for all proofs).
\begin{restatable}{proposition}{mainprop}
Let $\f^\mathrm{CGD}=\argmin\{R_\lambda(\f') \with \f' \in \range [\Q_{t-1}; \g]\}$ be the output vector of the ideal corrective gradient descent update in round $t$ and 
\begin{equation}\label{eq:ogbobj}
    q = \argmax_{q \in \cQ} |\g_\perp^T \q|/\|\q_\perp\|
\end{equation}
where for $\vecv \in \R^n$ we denote by $\vecv_\perp$ its projection onto the orthogonal complement of $\range \, \Q_{t-1}$.
Then the output vector $\f_q=\argmin\{R_\lambda(\f') \with \f' \in \range[\Q_{t-1}; \q]\}$ dominates the optimal geometric approximation to $\f$ with respect to $\cQ$, i.e.,
$R_\lambda(\f_q) \leq R_\lambda(\tilde{\f}^\mathrm{CGD}_\cQ)$.
\label{prop:main}
\end{restatable}
The right hand side of Eq.~\eqref{eq:ogbobj} is undefined for redundant query vectors $\q$ that lie in $\range\, \Q$ and therefore have $\|\q_\perp\|=0$.
Therefore we add a small positive value $\epsilon$ to the denominator, which can be considered an alternative \defemph{regularization parameter}.
With this we define the \defemph{orthogonal gradient boosting} objective function as
\begin{equation}
    \ogbobj(q) = |\inner{\g_\perp}{\q}|/(\|\q_\perp\|+\epsilon) \enspace .
    \label{eq:obj_final}
\end{equation}

This function measures the cosine of the angle between the gradient vector and the orthogonal projection of a candidate query vector $\q$. This is in contrast to the standard gradient boosting objective, which considers the angle of the unprojected query vector instead. 
In the example in Fig.~\ref{fig:subspace} we can observe that this difference leads to successfully identifying the best approximating subspace, as guaranteed by Prop.~\ref{prop:main}.
In that example, this corresponds to a factor of 24 improvement over the empirical risk produced by the gradient boosting objective function. 
In fact, the potential advantage of orthogonal gradient boosting is unbounded relative to both the gradient boosting and the gradient sum objective function.
\begin{restatable}{proposition}{advantageprop}\label{prop:advantage}
    There is one-dimensional input data $\X\in \R^{5\times1}$ and a sequence of output data, $\y^{(m)}\in \R^5$, such that for the empirical risk of three-element rule ensembles with corrective weight update we have $
    \lim_{m\rightarrow\infty}{R\left(f_\ogb\left(\X, \y^{(m)}\right)\right)}=0,
    $ but
    $$\lim_{m\rightarrow\infty}{\!\!\!R\!\left(f_\gb\!\left(\X, \y^{(m)}\right)\right)}\!=\!
    \lim_{m\rightarrow\infty}\!\!\!{R\!\left(f_\gs\!\left(\X, \y^{(m)}\right)\right)}\!=\!\infty$$
    where $f_{*}(\X, \y^{(m)})$ denotes the three-element rule ensemble trained by $\obj_{*}$.
\end{restatable}

\begin{proof}[Proof sketch]
    We define $\X$ as $(1,2,3,4,5)$
and $\y^{(m)}=(-\alpha_m-\epsilon_m, \alpha_m, -3\alpha_m-\epsilon_m, \alpha_m+\epsilon_m, 2\alpha_m+\epsilon_m)$,
where $\alpha_m, \epsilon_m \in \R$ are two arbitrary sequences with $\alpha_m\rightarrow\infty$ and $\epsilon_m\rightarrow0$.
The risks of the rule ensembles generated by 
$\gsobj$ is $R\left(f_\gs\left(\X, \y^{(m)}\right)\right)=3(3\alpha_m+\epsilon_m)^2/8$ or
$R\left(f_\gs\left(\X, \y^{(m)}\right)\right)=2\left(6\alpha_m^2+2\alpha_m\epsilon_m+\epsilon_m^2\right)/5.$
The risk of the rule ensemble generated by 
$\gbobj$ is $R\left(f_\gb\left(\X, \y^{(m)}\right)\right)=3\alpha_m^2/2.$
Therefore, $\lim_{m\rightarrow\infty}{\!\!\!R\!\left(f_\gb\!\left(\X, \y^{(m)}\right)\right)}\!=\!
    \lim_{m\rightarrow\infty}\!\!\!{R\!\left(f_\gs\!\left(\X, \y^{(m)}\right)\right)}\!=\!\infty$. 
However, the risk of the rule ensemble generated by $\ogbobj$ is $R\left(f_\ogb\left(\X, \y^{(m)}\right)\right)=3\epsilon_m^2/5$, whose limit is 0.

\end{proof}

\begin{algorithm}[tb]
   \caption{Corrective Orthogonal Boosting}
   \label{alg:fcogb}
\begin{algorithmic}
   \STATE {\bfseries Input:} dataset $(\x_i, y_i)_{i=1}^{n}$, number of rules $k$
   \vspace{0.1cm}
   \STATE $\sigma_j = \mathrm{argsort}(\{x_{i,j} \with 1 \leq i \leq n\})$ \textbf{for} $j \in \{1,\dots, p\}$
   \STATE $f^{(0)}=\beta_0=\argmin \{R_\lambda(\beta\mathbf{1}) \with \beta \in \R\}$
   \FOR{$t=1$ {\bfseries to} $k$}
   \STATE $\g_t=\left(\partial l(f^{(t-1)}(\x_i), y_n)/\partial f^{(t-1)}(\x_i)\right)_{i=1}^n$
   \STATE $q_t = \mathrm{ogb\_baselearner}(\g_t, \matO_t)$
   \STATE $\q_{t, \perp} = \q - \matO_t\matO_t^T\q$
   \STATE $\veco_t=\q_{t\perp}/\|\q_{t\perp}\|$ and $\matO_{t}=[\matO_{t-1}; \veco_t]$
   \STATE $\weights_t = \argmin\{R_\lambda(\beta_0+[\q_1, \dots, \q_t]\weights) \with {\weights \in \R^t}\}$ 
   \STATE $f^{(t)}(\cdot)=\beta_0 + \beta_{t, 1}q_1(\cdot)+\dots+\beta_{t, t}q_t(\cdot)$ 
   \ENDFOR
   \vspace{0.1cm}
   \STATE {\bfseries Output:} $f^{(0)}, \dots, f^{(k)}$
\end{algorithmic}
\end{algorithm}

\begin{algorithm}[tb]
   \caption{OGB Base Learner}
   \label{alg:baselearner}
\begin{algorithmic}
   \STATE {\bfseries Input:} data $(\x_i)_{i=1}^{n}$, grad. $\g$, orthonorm. $\matO \in \R^{n,t}$
   \vspace{0.1cm}
   
   \STATE $\g_\perp \leftarrow \g-\matO\matO^T\g$
   {\color{red} \STATE $\varphi \leftarrow \mathrm{argsort}(\g_\perp)$}
   \STATE $v_*, q_*, I_*\leftarrow(0, 1, \{1, \dots n\})$
   \STATE $\cB \leftarrow \text{empty priority queue with size limit } w$
   \STATE $\mathrm{enqueue}(\cB, (v_*, q_*, I_*))$
   \WHILE{not $\mathrm{empty}(\cB)$}
   \STATE $\cB' \leftarrow \text{empty priority queue with size limit } w$
   \FOR{$(v, q, I) \in \cB$}
   {\color{red}
   \STATE $\varphi' \leftarrow \mathrm{filter}(\varphi, I)$
   \STATE $b_+ \leftarrow \max \mathrm{prefix\_values}(\g_\perp, \matO, \varphi')$
   \STATE $b_- \leftarrow \max \mathrm{prefix\_values}(\g_\perp, \matO, \mathrm{inverted}(\varphi'))$
   \IF{$\max\{b_+, b_-\} \leq v_*$ \textbf{or} $\mathrm{redundant}(q, I)$}
   \STATE \textbf{continue}
   \ENDIF
   }
   \FOR{$j \in \{1, \dots, d\}$}
   \FOR{$s \in \{-1,1\}$}
   \STATE $\sigma \leftarrow \mathrm{filter}(\sigma_j, I)$
   \STATE \textbf{if} $s=-1$ \textbf{then} $\sigma \leftarrow \mathrm{inverted}(\sigma)$
   \STATE $\mathbf{v} = \mathrm{prefix\_values(\g_\perp, \matO, \sigma)}$
   \FOR{$i \in \{1, \dots, |\sigma|\}$}
   \STATE $q' \leftarrow q\delta(sx_{0, j} \leq sx_{\sigma(i), j})$
   \IF{$v' > v_*$}
   \STATE $v_*, q_* \leftarrow v', q'$
   \ENDIF
   \STATE $\mathrm{enqueue}(\cB', (v', q', \{\sigma(1), \dots, \sigma(i)\}))$
   \ENDFOR
   \ENDFOR
   \ENDFOR
   \ENDFOR
   \STATE $\cB \leftarrow \cB'$
   \ENDWHILE
   \vspace{0.1cm}
   \STATE {\bfseries Output:} $q_*$
\end{algorithmic}
\end{algorithm}

\begin{algorithm}[tb]
\caption{Incremental Prefix Values}
\label{alg:incremental_prefix}
\begin{algorithmic}
   \STATE {\bfseries Input:} proj. grad. $\g_\perp$, orthon. $\matO \in \R^{n,t}$, order $\sigma$
   \vspace{0.1cm}
 \STATE $G, N_1, \dots, N_t \leftarrow (0, 0, \dots, 0)$
   \FOR{$i \in \{1, \dots, |\sigma|\}$}
   \STATE $G \leftarrow G + g_{\perp, \sigma(i)}$
   \FOR{$k \in \{1, \dots, t\}$}
   \STATE $N_k \leftarrow N_k + o_{k, \sigma(i)}$
   \ENDFOR
   \STATE $v_i = |G|/\left(\sqrt{i^2 -\sum_{k=1}^t N^2_k}+\epsilon\right)$
   \ENDFOR
   \vspace{0.1cm}
   \STATE {\bfseries Output:} $(v_1, \dots, v_l)$
\end{algorithmic}
\end{algorithm}

\subsection{Efficient Implementation}
To develop an efficient optimization algorithm for the orthogonal gradient boosting objective, 
we recall that projections $\q_\perp$ on the orthogonal complement of $\range\,\Q$ can be naively computed via $\q_\perp=\q-\Q((\Q^T\Q)^{-1}(\Q^T\q))$
where we placed the parentheses to emphasize that only matrix-vector products are involved in the computation---at least once the inverse of the Gram matrix $\Q^T\Q$ is computed.
This approach allows to compute projections, and thus objective values, in time $O(nt+t^2)$ per candidate query after an initial preprocessing per boosting round of cost $O(t^2n + t^3)$.

In a first step, this naive approach can be improved by maintaining an orthonormal basis of the range of the query matrix throughout the boosting rounds, resulting in a Gram-Schmidt-type procedure. Since the projections $\q_{\perp}$ of the selected query output vectors already form an orthogonal basis of $\range\, \Q$ this only requires normalization with negligible additional cost. Formally, by storing $\veco_t = \q_{t\perp} / \|\q_{t\perp}\|$ in all boosting rounds $t$, subsequent projections can be computed via
$
    \q_{t+1, \perp} = \q_{t+1} - \matO_t(\matO_t^T\q_{t+1})
$
where $\matO_t=[\veco_1, \dots, \veco_t]$.
This reduces the computational complexity per candidate query to $O(tn)$ without  additional preprocessing.



As mentioned above, to achieve an acceptable scalability with $n$, we need to reduce this complexity further for evaluating sequences of candidate queries, as required by either beam or branch-and-bound search.
In particular we need to solve the  prefix 
optimisation problem \eqref{eq:prefix_opt}, which translates for our objective function to finding:
\begin{equation}
 i_* = \argmax_{i \in \{1, \dots, l\}} \frac{|\inner{\g_\perp}{\q^{(i)}}|}{\|\q^{(i)}_\perp\| + \epsilon} 
 \label{eq:best_prefix}
\end{equation}
given an ordered sub-selection $\sigma$ where $\q^{(0)}=\0$ and $\q^{(i)}=\q^{(i-1)}+\e_{\sigma(i)}$.
The following proof shows how the computational complexity for solving~\eqref{eq:best_prefix} can be substantially reduced compared to the direct approach above. 
It uses an incremental computation of projections that works directly on the available orthonormal basis vectors $\veco$ instead of computing matrix-vector products or, even worse, the whole projection matrix.
\begin{restatable}{proposition}{efficientprop}
    Given a gradient vector $\g \in \R^n$, an orthonormal basis $\veco_1, \dots, \veco_t \in \R^n$ of the subspace spanned by the queries of the first $t$ rules, and a sub-selection of $l$ candidate points $\sigma\from [l] \to [n]$, the best prefix selection problem~\eqref{eq:best_prefix} can be solved in time $O(tl)$.
    \label{thm:fast_prefix}
\end{restatable}
\begin{proof}[Proof sketch]
    We can write the objective value of prefix $i$ in terms of incrementally computable quantities:
    \begin{align*}
        &\frac{\left|\inner{\g_\perp}{\q^{(i)}}\right|}{\|\q^{(i)}_\perp\| + \epsilon} = \frac{\left|\inner{\g_\perp}{\q^{(i)}}\right|}{\sqrt{\|\q^{(i)}\|^2 - \|\q^{(i)}_\parallel\|^2} + \epsilon} \\= &\frac{\left|\inner{\g_\perp}{\q^{(i)}}\right|}{\sqrt{\|\q^{(i)}\|^2 - \sum_{k=1}^t \|\veco_k \veco_k^T\q^{(i)}\|^2} + \epsilon}
        \enspace .
    \end{align*}
    In particular, the $t$ sequences of norms $\|\veco_k \veco_k^T\q^{(i)}\|$ can be computed in time $O(l)$ via cumulative summation of the $k$-th basis vector elements in the given order:
    \begin{align*}
        \|\veco_k\veco_k^T\q^{(i)}\|
        =\|\veco_k\| \left|\sum_{j=1}^i \inner{\veco_k}{\e_{\sigma(j)}}\right|
        =\left|\sum_{j=1}^i o_{k, \sigma(j)} \right|
    \end{align*}
\end{proof}
This result implies that the computational complexity of maximizing the objective function in boosting iteration $t$ depends linearly on $t$, which introduces an asymptotic overhead of a factor of $k$ over the previous objective functions when running boosting for $k$ iterations.
However, as we will discuss below, in practice for finite $k$ we typically find a much smaller overhead, e.g., of around $2$ for most data sets when running $10$ iterations.
This is due to additive preprocessing costs that are common to all objective functions and the tendency to select more general rules compared to GB/XGB.

We close this section with pseudocodes that summarize the main ideas of orthogonal gradient boosting, starting with the high level algorithm (Alg.~\ref{alg:fcogb}), followed by the base learner for the query optimization at each boosting round (Alg.~\ref{alg:baselearner}), and closing with the fast incremental objective function computation (Alg.~\ref{alg:incremental_prefix}).
This specific version of Alg.~\ref{alg:baselearner} uses breadth-first-search for ease of exposition. Other search orders can be implemented. Moreover, it uses the following operations and notations for ordered sub-collections of data point indices $\sigma: \{1, \dots, m\} \to \{1, \dots n\}$, which can be represented by integer arrays of length $m$. The cardinality symbol refers to the size of the represented sub-collection, i.e., $|\sigma|=m$ . The operation $\mathrm{filter}(\sigma, I)$ refers to the order $(\sigma(i) \with 1 \leq i \leq |\sigma|, \sigma(i) \in I)$, which can be computed in time $O(m)$ assuming $I$ is given as a Boolean array. Finally, $\mathrm{inverted}(\sigma)$ refers to the inverted order $\sigma'(i)=\sigma(m-i+1)$.
For beam search the priority queue limit $w$ has to be set to some finite positive integer, where $w=1$ yields the standard greedy algorithm as special case.
The setting $w=\infty$ leads to branch-and-bound, for which the red lines implement the bounding part.
Importantly, the computed bounding function based on a prefix greedy optimization with respect to the gradient order is just a heuristic for $\ogbobj$.
Therefore, one might want to omit it for beam search. However, we find in extensive numerical experiments (see SI) that this approach leads to a 3/4-approximation algorithm with high probability.

\section{Empirical Evaluation}\label{experiments}

\begin{table}
\caption{Datasets with size $n$, dimensions $d$, and type: classification ($^+$), Poisson ($^*$), ordinary regression ( )}
\label{tb:data_info}
\begin{center}
\begin{scriptsize}
\begin{sc}
\begin{tabular}{@{\hskip 0.02in}l@{\hskip 0.02in}c@{\hskip 0.01in}c|@{\hskip 0.02in}l@{\hskip 0.02in}c@{\hskip 0.02in}c|@{\hskip 0.02in}l@{\hskip 0.02in}c@{\hskip 0.02in}c}
\toprule
Dataset & $d$ & $n$ & Dataset & $d$ & $n$ & Dataset & $d$ & $n$ \\
\hline
ships$^*$ & 4 & 34 & breast$^+$ & 30 & 569 & gender$^+$ & 20 & 3168 \\
gdp & 1 & 35 & tic-tac-toe$^+$ & 27 & 958 & digits5$^+$ & 64 & 3915 \\
smoking$^*$ & 2 & 36 & titanic$^+$ & 7 & 1043 & friedman3 & 4 & 5000 \\
covid vic$^*$ & 4 & 85 & insurance & 6 & 1338 & demograph. & 13 & 6876 \\
bicycle$^*$ & 4 & 122 & banknote$^+$ & 4 & 1372 & tel. churn$^+$ & 18 & 7043 \\
iris$^+$ & 4 & 150 & wage & 5 & 1379 & friedman2 & 4 & 10000 \\
wine$^+$ & 13 & 178 & ibm hr$^+$ & 32 & 1470 & magic$^+$ & 6 & 16327 \\
covid$^*$ & 2 & 225 & red wine & 11 & 1599 & videogame & 5 & 27820 \\
happiness & 8 & 315 & life expect. & 21 & 1649 & suicide rate & 5 & 27820 \\
liver$^+$ & 6 & 345 & used cars & 4 & 1770 & adult$^+$ & 11 & 30162 \\
diabetes & 10 & 442 & mobile price & 20 & 2000 &  &  &  \\
boston & 13 & 506 & friedman1 & 10 & 2000 &  &  & \\
\bottomrule
\end{tabular}
\end{sc}
\end{scriptsize}
\end{center}
\end{table}
In this section, we present empirical results comparing the proposed corrective orthogonal gradient boosting (COB) to the standard gradient boosting algorithms~\citep{dembczynski2010ender} using greedy optimization of $\gbobj$ (SGB) and $\gsobj$ (SGS) with stagewise weight update, to extreme gradient boosting~\citep{boley2021better} using branch-and-bound optimisation of $\xgbobj$ with stagewise weight update (SXB), and finally to SIRUS \citep{benard2021interpretable} as the state-of-the-art generate-and-filter approach.
We investigate the risk/complexity trade-off, the affinity to select general rules, as well as the computational complexity.
The datasets used are those of \citet{boley2021better} augmented by three additional classification datasets from the UCI machine learning repository and, to introduce a novel modelling task to the rule learning literature, five counting regression datasets from public sources. 
This results in a total of 34 datasets (13 for classification, 16 for regression, and 5 for counting/Poisson regression, see Tab.~\ref{tb:data_info}).
The experiment code and further information about the datasets are available on GitHub (\url{https://github.com/fyan102/FCOGB}).
\begin{figure}[bt]
\centering
\begin{center}
\includegraphics[width=0.96\columnwidth]{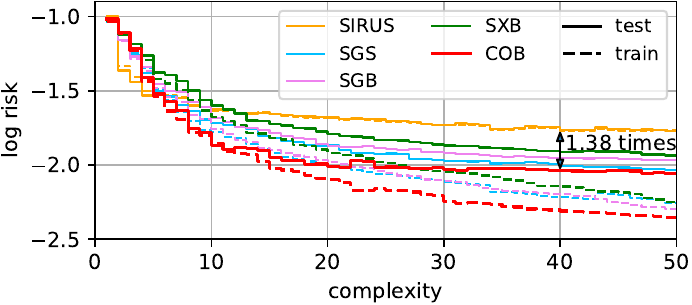}
\includegraphics[width=0.96\columnwidth]{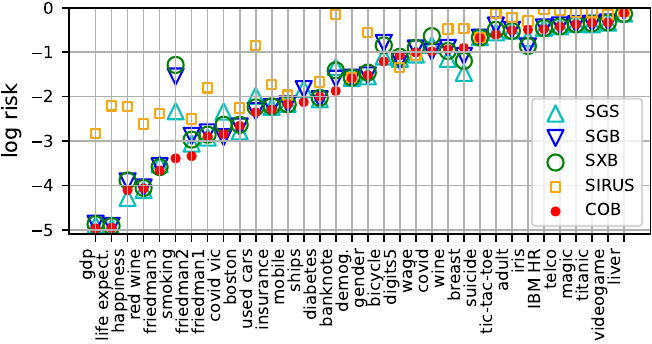}
\end{center}
\vskip -0.1in
\caption{Method log risks across complexity levels (top) and log test risks across datasets (bottom). 
}
\label{fig:experiments}
\end{figure}

All algorithms were run five times on all datasets using 5 random 80/20 train/test splits to calculate robust estimates of all considered metrics. 
For each boosting variant, we consider ensembles with $k=1,2,\dots$ rules until the complexity exceeds 50.
For each ensemble size $k$ (and each training split) a separate regularization parameter value $\lambda_k$ is found via an internal five-fold cross-validation as
\begin{equation*}
    \argmin\{R_\mathrm{CV}(f_\lambda^{(k)}):\lambda=10^a,a\in\{-2, -1, 0, 1 ,2\}\}
\end{equation*}
Note that with this procedure, the ensemble complexity is not necessarily monotone with the number of rules, i.e., because typically $\lambda_l >\lambda_k$ for $l > k$.
For SIRUS, we consider different ensembles by varying the minimum occurrence frequency $p_0$ of a rule (in trees of the random forest) required to be added to the ensemble, incrementally decreasing $p_0$ until complexity 50 is exceeded.


\paragraph{Complexity versus risk}
We firstly compare the complexity/risk trade-off of SGB, SGS, SXB and COB. 
Fig. \ref{fig:experiments} compares the log risks in terms of the complexity of rule ensembles generated by all methods. 
The log risks are used such that their difference indicate risk ratios between methods, as $\log(R_A/R_B)=\log(R_A)-\log(R_B)$.
Here, normalization is performed by the risk of the rule ensemble with a single empty ``offset'' rule.
The top part compares the risks  per complexity level averaged across all datasets. 
One can see that, for almost all complexity levels, COB has the smallest average training and test risks for all but the smallest complexity levels, where it is only beaten by SIRUS.
On the other hand, SIRUS is not competitive for complexity levels greater than 15.
The bottom part of Fig.~\ref{fig:experiments} compares the risk per dataset  averaged across all considered cognitive complexity levels from 1 to 50.
COB generates rules which have the smallest test risks for 23 out of 34 datasets (and the smallest training risks for 26 out of 34 datasets, see Fig.~\ref{fig:log_train_risk} in SI).
Moreover, COB occasionally outperforms the second-best algorithm by a wide margin (\textit{tic-tac-toe}, \textit{banknote}, \textit{insurance}, \textit{friedman2}, \textit{used cars}, \textit{smoking}).
One-sided paired t-tests at significance level 0.05 (Bonferroni-corrected for 8 hypotheses, 4 for training risks and 4 for test risks) reveal that COB significantly outperforms all other methods for a random dataset by a margin of at least 0.001 average training risk and testing risk. 

\begin{figure}[t!]
\vskip 0.1in
\begin{center}
\includegraphics[width=\columnwidth]{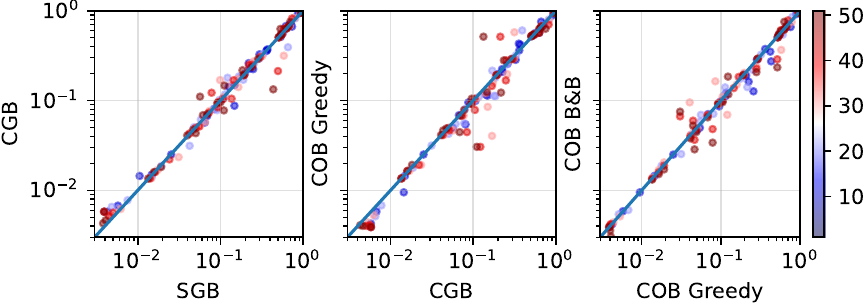}
\label{fig:ablation}
\vskip -0.25in
\caption{Comparison of risks with different complexity levels between Stepwise Gradient Boosting (SGB), Corrective Gradient Boosting (CGB), COB using Greedy search and COB using Branch-and-bound search. The colours represent the complexity of the rule ensembles.
}
\label{fig:ablation_study}
\end{center}
\vskip -0.2in
\end{figure}
\begin{figure}[t!]
\begin{center}
\includegraphics[width=1\columnwidth]{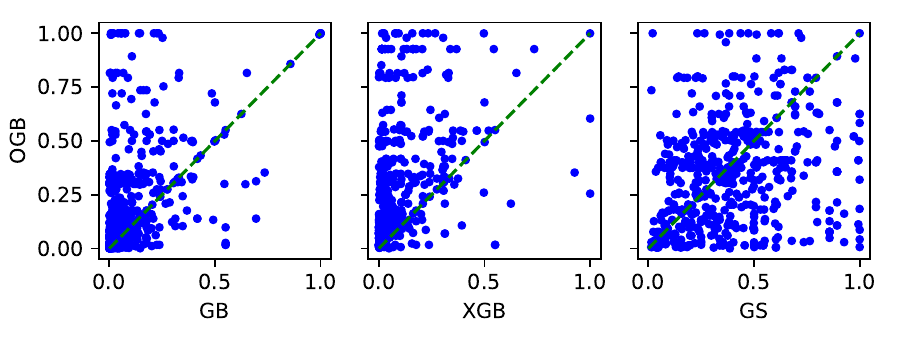}
\label{fig:time_tic}
\vskip -0.3in
\caption{Coverage rate of the rules generated by Gradient Boosting, XGBoost, Gradient Sum versus OGB. 
}
\label{fig:compare_coverage}
\end{center}
\vskip -0.2in
\end{figure}

\begin{figure}[t!]
\vskip 0.1in
\centering
\begin{center}
\includegraphics[width=\columnwidth]{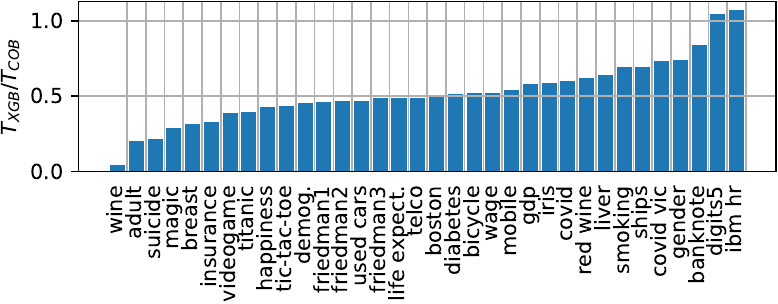}
\includegraphics[width=\columnwidth]{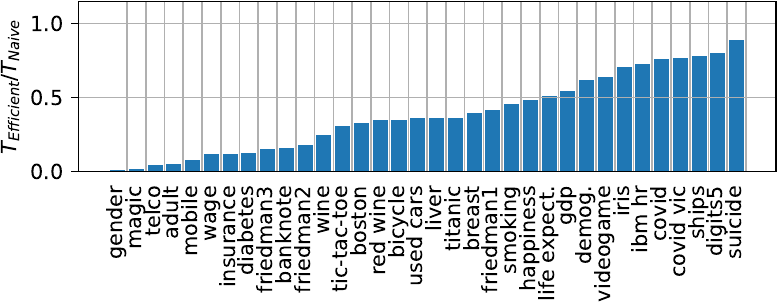} 
\end{center}
\vskip -0.2in
\caption{Running time ratio of SXB and COB (top) and naive and efficient opt. of COB (bottom).}
\label{fig:computation_time}
\end{figure}
\paragraph{Ablation study}
After seeing that the developed rule learning approach yields overall significant advantages, we next investigate the necessity of its individual components. 
Fig.~\ref{fig:ablation_study} compares the risk of various intermediate variants of COB for all investigated datasets and representative complexity levels.
As can be seen, applying just corrective weights updates does not substantially improve  standard gradient boosting with stagewise updates and greedy rule optimization, with only 57.7\% of the considered ensembles improved.
Correspondingly, a ones-sided paired t-test at significance level 0.05 does not reject the null hypothesis that the expected risks differ at all, i.e., without a margin,
between these two methods (for a random dataset and complexity level). In contrast, using the orthogonal objective function $\ogbobj$ in conjunction with corrective updates improves the risk values in 78.9\% of the cases (compared to using $\gbobj$ with corrective update), and optimizing $\ogbobj$ with branch-and-bound search instead of greedy search improves the performance further (63.4\% of ensembles are improved).
Both improvements are significant, i.e., two further one-sided paired t-tests at significance level 0.05, with Bonferroni correction for three tests, reject the null hypotheses that the difference in expected risks are less than 0.001.
In terms of the ensemble sizes there are no clear patterns in terms of the number of wins and losses per step. However, we can observe that larger differences are tend to be attained on larger ensembles.


\paragraph{Effect on Condition Complexity}
The objective functions primarily affect the risk / complexity trade-off through the number of rules required to reach a certain complexity level.
However, as a secondary effect, they might also have a different propensity to select complex rule conditions with many propositions, which our complexity measure also accounts for.
While condition complexity is not directly assessed by any of the investigated objective functions, it is related to the condition's \defemph{coverage rate}, i.e., the relative number of selected data points $\|\q\|_0/n$, because less complex conditions tend to have a higher coverage rate.
Here we assess differences in the coverage rate in the following way:
for each alternative objective function $\obj_\mathrm{alt}$ we fit rule ensembles $f_\mathrm{alt}^{(t)}$ for $t=1, \dots, 30$ for all considered datasets (using stagewise weight updates). For each $t$, we then compare the coverage rate of the next condition ($q_{t+1}$) produced by $\obj_\mathrm{alt}$ to the one that is produced by $\ogbobj$ based on the same previous model $f_\mathrm{alt}^{(t)}$.
Importantly, we use branch-and-bound for all the objectives to avoid confounding through sub-optimal greedy solutions.
As shown in Fig.~\ref{fig:compare_coverage}, 91.0\% of conditions identified by $\ogbobj$ cover more data points than those identified by $\xgbobj$, and similarly 78.6\% of the $\ogbobj$ conditions cover more data points than those generated by $\gbobj$. In contrast, only 46.1\% of conditions identified by $
\ogbobj$ cover more data points than the gradient sum objective $\gsobj$.
These results are aligned with the theoretical expectation in terms of the influence of the coverage on the objective values where gradient sum is completely unaffected, whereas orthogonal gradient boosting has a denominator that tends to grow with coverage albeit less than the one of gradient boosting. 

\paragraph{Computation time}
Finally, we investigate the computational overhead for generating rule ensembles with the proposed COB method with branch-and-bound search. Here, we use XGB as primary benchmark, because it uses the same more expensive search. Further, we investigate the effect of the efficient incremental computation of $\ogbobj$ via Alg.~\ref{alg:incremental_prefix} to assess its necessity. 
For both comparisons, we consider the time it takes for each test dataset to compute the largest rule ensemble generated in the main experiment (the one that exceeds complexity 50).
The top of Fig~\ref{fig:computation_time}.
compares the efficient implementation of COB to XGB. We can see that the costs are in the same order of magnitude for almost all datasets. For 17 of the 34 datasets the overhead is within a factor of 2. For all but one extreme case (\textit{wine}, overhead factor 26) the overhead is within a factor of 5.
Comparing the two implementations of COB, the bottom of Fig.~\ref{fig:computation_time} the efficient implementation uses less than half of the time of the naive implementation for 24 out of 34 datasets. 
Particularly, for the datasets \textit{gender} and \textit{magic}, the efficient implementation improves the running time more than 50 times. 
Overall, the results confirm that branch-and-bound search is a practical algorithm in absolute terms: For 23 benchmark dataset, COB is able to finish training a model of complexity of 50 within one minute. Most of the  other experiments run within 15 minutes except one dataset (telco churn) which require longer running time. See SI~\ref{sec:SI:eval} for further details and comparisons. 

\section{Conclusion}\label{conclustions}
The proposed fully corrective orthogonal boosting approach is a worthwhile alternative to previously published boosting variants for rule learning, especially when targeting a beneficial risk-complexity trade-off and an overall small number of rules.
The present work provided a relatively detailed theoretical analysis of the newly developed rule objective function. However, some interesting questions were left open. 
While the presorting-based approach to the bounding function performs extremely well in synthetic experiments, 
a theoretical approximation guarantee for this algorithm has yet to be derived.
Another interesting direction for future work is the extension of the introduced approximating subspace paradigm to the extreme gradient boosting approach, which, due to the utilization of higher order information, should principally be able to produce even better risk-complexity trade-offs.

\subsubsection*{Acknowledgements}
The authors thank Daniel F. Schmidt for valuable discussions and the anonymous reviewers for their constructive feedback. This work was supported by the Australian Research Council (DP210100045).

\bibliography{references}
\begin{appendix}
\onecolumn
\section{Full Proofs and Additional Formal Statements}
\mainprop*

\begin{proof}
    Let $\Q=\Q_{t-1}$ and $\f=[\Q; \g]\alphavec$ and $\tilde{\f}=[\Q; \q]\weights$ for some arbitrary coefficient vectors $\alphavec, \weights \in \R^t$. Denoting by $\vecv_\parallel$ the projection of  $\vecv \in \R^n$ onto the column space of $\Q$ and its orthogonal complement by $\vecv_\perp$, we can decompose the squared norm of the difference $\f-\tilde{\f}$ as
    \begin{align*}
        \|\f-\tilde{\f}\|^2 &= \|[\Q; \g]\alphavec - [\Q; \q]\weights\|^2\\
        &=\|[\Q; \g_\parallel + \g_\perp]\alphavec - [\Q; \q_\parallel + \q_\perp]\weights\|^2\\
        &=\|[\Q; \g_\parallel]\alphavec + \alpha_t\g_\perp - [\Q; \q_\parallel]\weights + \weight_t\q_\perp \|^2\\
        &=\|[\Q; \g_\parallel]\alphavec - [\Q; \q_\parallel]\weights\|^2 + \|\alpha_t \g_\perp - \beta_t \q_\perp\|^2
    \end{align*}
    where the last step follows from the Pythagorean theorem and the fact that $\alpha_t \g_\perp - \beta_t \q_\perp$ is an element from the orthogonal complement of $\range [\Q; \g_\parallel]=\range[\Q; \q_\parallel]=\range\, \Q$.
    The equality of these ranges also implies that $\beta_1, \dots, \beta_{t-1}$ can, for all choices of $\beta_t$, be chosen such that the left term of the error decomposition is $0$. Setting $\gamma=\beta_t/\alpha_t$, it follows for the squared projection error of $\f$ onto $\range [\Q, \q]$ that
    \begin{align*}
        \min_{\weights \in \R^t} \|\f-\tilde{\f}\|^2 &= \min_{\beta \in \R^t} \|\alpha_t \g_\perp - \beta_t \h_\perp\|^2\\
        &=\min_{\gamma \in \R^t} \alpha^2_t \|\g_\perp - \gamma \q_\perp\|^2\\
        &=\min_{\gamma \in \R^t} \alpha_t^2 (\|\g_\perp\|^2-2\gamma \inner{\q_\perp}{\g_\perp} + \gamma^2\|\q_\perp\|^2) \\
        \intertext{and plugging in the minimizing $\gamma=\inner{\q_\perp}{\g_\perp}/\|\q_\perp\|^2$ }
        &=\alpha^2(\|\g_\perp\| - (\inner{\g_\perp}{\q_\perp})^2/\|\q_\perp\|^2) \enspace ,
    \end{align*}
    from which, noting that $\inner{\g_\perp}{\q_\perp}=\inner{\g_\perp}{\q}$, it follows that a query that maximizes $|\inner{\g_\perp}{\q}|/\norm{\q_\perp}$ minimizes the projection error. Hence, by choosing $\f=\f^\mathrm{GCD}$ the ideal corrective gradient descent update, we have that $\tilde{\f}^\mathrm{GCD}_\cQ \in \range [\Q; \q]$ and by definition of $\f_q$ we have $R_\lambda(\f_q) \leq R_\lambda(\tilde{\f}^\mathrm{GCD}_\cQ)$ as required.
\end{proof}

\begin{table}[t]
\caption{Calculation of the objective functions for the first and the second query in proof of Proposition \ref{prop:advantage}}
\label{tb:prop2proof1}
\vskip 0.05in
\begin{center}
\begin{footnotesize}
\begin{sc}
\begin{tabular}{c@{\hskip 0.03in}c|@{\hskip 0.03in}c@{\hskip 0.05in}c@{\hskip 0.03in}|c@{\hskip 0.03in}c@{\hskip 0.03in}c@{\hskip 0.03in}c@{\hskip 0.03in}c@{\hskip 0.03in}c@{\hskip 0.03in}c}
\hline
 &  & \multicolumn{2}{c|}{1st query} & \multicolumn{7}{c}{2nd   query} \\
\multirow{-2}{*}{$\q$} & \multirow{-2}{*}{$\|\q\|$} & \begin{tabular}[c]{@{}c@{}}$\gsobj(\q)$\\      $\left|\q^T\g^{(0)}\right|$\end{tabular} & \begin{tabular}[c]{@{}c@{}}$\gbobj(\q)$\\      $\ogbobj(\q)$\end{tabular} & $\q_\bot$ & $\|\q_\bot\|$ & \begin{tabular}[c]{@{}c@{}}$\left|\q^T\g^{(1)}_\gb\right|$\\      $\left|\q^T_\bot\g^{(1)}_{\ogb\bot}\right|$\end{tabular} & \begin{tabular}[c]{@{}c@{}}$\gsobj(\q)$\\      1st case\end{tabular} & \begin{tabular}[c]{@{}c@{}}$\gsobj(\q)$\\      2nd case\end{tabular} & $\gbobj(\q)$ & $\ogbobj(\q)$ \\
\hline
$(1, 0, 0, 0, 0)$ & $1$ & $\alpha_m+\epsilon_m$ & $\alpha_m+\epsilon_m$ & $(1, 0, 0, 0, 0)$ & $1$ & $\alpha_m+\epsilon_m$ & $\cfrac{\epsilon_m}{3}$ & $\alpha_m+\epsilon_m$ & $\alpha_m+\epsilon_m$ & $\alpha_m+\epsilon_m$ \\
$(0, 1, 0, 0, 0)$ & $1$ & $\alpha_m$ & $\alpha_m$ & $(0, 1, 0, 0, 0)$ & $1$ & $\alpha_m$ & $2\alpha_m+\cfrac{2\epsilon_m}{3}$ & $\alpha_m$ & $\alpha_m$ & $\alpha_m$ \\
$(0, 0, 1, 0, 0)$ & $1$ & $3\alpha_m+\epsilon_m$ & {\color[HTML]{FE0000} \textbf{$3\alpha_m+\epsilon_m$}} & $(0, 0, 0, 0, 0)$ & $0$ & $0$ & $2\alpha_m+\cfrac{\epsilon_m}{3}$ & $3\alpha_m+\epsilon_m$ & $0$ & $0$ \\
$(0, 0, 0, 1, 0)$ & $1$ & $\alpha_m+\epsilon_m$ & $\alpha_m+\epsilon_m$ & $(0, 0, 0, 1, 0)$ & $1$ & $\alpha_m+\epsilon_m$ & $\alpha_m+\epsilon_m$ & $\cfrac{\alpha_m}{2}$ & $\alpha_m+\epsilon_m$ & $\alpha_m+\epsilon_m$ \\
$(0, 0, 0, 0, 1)$ & $1$ & $2\alpha_m+\epsilon_m$ & $2\alpha_m+\epsilon_m$ & $(0, 0, 0, 0, 1)$ & $1$ & $2\alpha_m+\epsilon_m$ & $2\alpha_m+\epsilon_m$ & $\cfrac{\alpha_m}{2}$ & $2\alpha_m+\epsilon_m$ & $2\alpha_m+\epsilon_m$ \\
$(1, 1, 0, 0, 0)$ & $\sqrt{2}$ & $\epsilon_m$ & $\cfrac{\epsilon_m}{\sqrt{2}}$ & $(1, 1, 0, 0, 0)$ & $\sqrt{2}$ & $\epsilon_m$ & $2\alpha_m+\cfrac{\epsilon_m}{3}$ & $\epsilon_m$ & $\cfrac{\epsilon_m}{\sqrt{2}}$ & $\cfrac{\epsilon_m}{\sqrt{2}}$ \\
$(0, 1, 1, 0, 0)$ & $\sqrt{2}$ & $2\alpha_m+\epsilon_m$ & $\cfrac{2\alpha_m+\epsilon_m}{\sqrt{2}}$ & $(0, 1, 0, 0, 0)$ & $1$ & $\alpha_m$ & $\cfrac{\epsilon_m}{3}$ & $2\alpha_m+\epsilon_m$ & $\cfrac{\alpha_m}{\sqrt{2}}$ & $\alpha_m$ \\
$(0, 0, 1, 1, 0)$ & $\sqrt{2}$ & $2\alpha_m$ & $\sqrt{2}\alpha_m$ & $(0, 0, 0, 1, 0)$ & $1$ & $\alpha_m+\epsilon_m$ & $\alpha_m-\cfrac{2\epsilon_m}{3}$ & $\cfrac{7\alpha_m}{2}+\epsilon_m$ & $\cfrac{\alpha_m+\epsilon_m}{\sqrt{2}}$ & $\alpha_m+\epsilon_m$ \\
$(0, 0, 0, 1, 1)$ & $\sqrt{2}$ & {\color[HTML]{FE0000} \textbf{$3\alpha_m+2\epsilon_m$}} & $\cfrac{3\alpha_m+2\epsilon_m}{\sqrt{2}}$ & $(0, 0, 0, 1, 1)$ & $\sqrt{2}$ & $3\alpha_m+2\epsilon_m$ & $3\alpha_m+2\epsilon_m$ & $0$ & {\color[HTML]{FE0000} \textbf{$\cfrac{3\alpha_m+2\epsilon_m}{\sqrt{2}}$}} & $\cfrac{3\alpha_m+2\epsilon_m}{\sqrt{2}}$ \\
$(1, 1, 1, 0, 0)$ & $\sqrt{3}$ & {\color[HTML]{FE0000} \textbf{$3\alpha_m+2\epsilon_m$}} & $\cfrac{3\alpha_m+2\epsilon_m}{\sqrt{3}}$ & $(1, 1, 0, 0, 0)$ & $\sqrt{2}$ & $\epsilon_m$ & $0$ & $3\alpha_m+2\epsilon_m$ & $\cfrac{\epsilon_m}{\sqrt{3}}$ & $\cfrac{\epsilon_m}{\sqrt{2}}$ \\
$(0, 1, 1, 1, 0)$ & $\sqrt{3}$ & $\alpha_m$ & $\cfrac{\alpha_m}{\sqrt{3}}$ & $(0, 1, 0, 1, 0)$ & $\sqrt{2}$ & $2\alpha_m+\epsilon_m$ & $\alpha_m+\cfrac{4\epsilon_m}{3}$ & $\cfrac{5\alpha_m}{2}+\epsilon_m$ & $\cfrac{2\alpha_m+\epsilon_m}{\sqrt{3}}$ & $\cfrac{2\alpha_m+\epsilon_m}{\sqrt{2}}$ \\
$(0, 0, 1, 1, 1)$ & $\sqrt{3}$ & $\epsilon_m$ & $\cfrac{\epsilon_m}{\sqrt{3}}$ & $(0, 0, 0, 1, 1)$ & $\sqrt{2}$ & $3\alpha_m+2\epsilon_m$ & $\alpha_m+\cfrac{5\epsilon_m}{3}$ & $3\alpha_m+\epsilon_m$ & $\cfrac{3\alpha_m+2\epsilon_m}{\sqrt{3}}$ & $\cfrac{3\alpha_m+2\epsilon_m}{\sqrt{2}}$ \\
$(1, 1, 1, 1, 0)$ & $2$ & $2\alpha_m+\epsilon_m$ & $\alpha_m+\cfrac{\epsilon_m}{2}$ & $(1, 1, 0, 1, 0)$ & $\sqrt{3}$ & $\alpha_m$ & $\alpha_m+\epsilon_m$ & {\color[HTML]{FE0000} \textbf{$\cfrac{7\alpha_m}{2}+2\epsilon_m$}} & $\cfrac{\alpha_m}{2}$ & $\cfrac{\alpha_m}{\sqrt{3}}$ \\
$(0, 1, 1, 1, 1)$ & $2$ & $\alpha_m+\epsilon_m$ & $\cfrac{\alpha_m+\epsilon_m}{2}$ & $(0, 1, 0, 1, 1)$ & $\sqrt{3}$ & $4\alpha_m+2\epsilon_m$ & {\color[HTML]{FE0000} \textbf{$3\alpha_m+\cfrac{7\epsilon_m}{3}$}} & $2\alpha_m+\epsilon_m$ & $2\alpha_m+\epsilon_m$ & {\color[HTML]{FE0000} \textbf{$\cfrac{4\alpha_m+\epsilon_m}{\sqrt{3}}$}} \\
$(1, 1, 1, 1, 1)$ & $\sqrt{5}$ & $0$ & $0$ & $(1, 1, 0, 1, 1)$ & $2$ & $3\alpha_m+\epsilon_m$ & $3\alpha_m+2\epsilon_m$ & $3\alpha_m+2\epsilon_m$ & $\cfrac{3\alpha_m+\epsilon_m}{\sqrt{5}}$ & $\cfrac{3\alpha_m+\epsilon_m}{2}$
\\
\hline
\end{tabular}
\end{sc}
\end{footnotesize}
\end{center}
\vskip -0.1in
\end{table}

\begin{table}[t]
\caption{Calculation of the objective functions for the third query in proof of Proposition \ref{prop:advantage}}
\label{tb:prop2proof2}
\vskip 0.05in
\begin{center}
\begin{footnotesize}
\begin{sc}
\begin{tabular}{c@{\hskip 0.05in}c@{\hskip 0.05in}c@{\hskip 0.05in}c@{\hskip 0.05in}c@{\hskip 0.05in}c@{\hskip 0.05in}c@{\hskip 0.05in}c@{\hskip 0.05in}c@{\hskip 0.05in}c}
\hline
$\q$ & $\|\q\|$ & $\q_\bot$ & $\|\q_\bot\|_2$ & $\left|\q^T_\bot\g^{(1)}_{\ogb\bot}\right|$ & $\left|\q^T\g^{(2)}_\gb\right|$ & \begin{tabular}[c]{@{}c@{}}$\gsobj(\q)$\\      1st case\end{tabular} & \begin{tabular}[c]{@{}c@{}}$\gsobj(\q)$\\      2nd case\end{tabular} & $\gbobj(\q)$ & $\ogbobj(\q)$ \\
\hline
$(1, 0, 0, 0, 0)$ & $1$ & $(1, 0, 0, 0, 0)$ & $1$ & $\alpha_m+\epsilon_m$ & $\alpha_m+\epsilon_m$ & $\cfrac{3\alpha_m+\epsilon_m}{4}$ & $\cfrac{3\epsilon_m}{7}$ & {\color[HTML]{FE0000} \textbf{$\alpha_m+\epsilon_m$}} & $\alpha_m+\epsilon_m$ \\
$(0, 1, 0, 0, 0)$ & $1$ & $\left(0, \cfrac{2}{3}, 0, -\cfrac{1}{3},   -\cfrac{1}{3}\right)$ & $\sqrt{\cfrac{2}{3}}$ & $\cfrac{\alpha_m+2\epsilon_m}{3}$ & $\alpha_m$ & $\cfrac{13\alpha_m+3\epsilon_m}{8}$ & {\color[HTML]{FE0000} \textbf{$2\alpha_m+\cfrac{4\epsilon_m}{7}$}} & $\alpha_m$ & $\cfrac{\alpha_m+2\epsilon_m}{\sqrt{6}}$ \\
$(0, 0, 1, 0, 0)$ & $1$ & $(0, 0, 0, 0, 0)$ & $0$ & $0$ & $0$ & $\cfrac{19\alpha_m+5\epsilon_m}{8}$ & $2\alpha_m+\cfrac{3\epsilon_m}{7}$ & $0$ & $0$ \\
$(0, 0, 0, 1, 0)$ & $1$ & $\left(0, -\cfrac{1}{3}, 0, \cfrac{2}{3},   -\cfrac{1}{3}\right)$ & $\sqrt{\cfrac{2}{3}}$ & $\cfrac{\alpha_m-\epsilon_m}{3}$ & $\cfrac{\alpha_m}{2}$ & $\cfrac{\alpha_m-\epsilon_m}{8}$ & $\cfrac{2\epsilon_m}{7}$ & $\cfrac{\alpha_m}{2}$ & $\cfrac{\alpha_m-\epsilon_m}{\sqrt{6}}$ \\
$(0, 0, 0, 0, 1)$ & $1$ & $\left(0, -\cfrac{1}{3}, 0, -\cfrac{1}{3},   \cfrac{2}{3}\right)$ & $\sqrt{\cfrac{2}{3}}$ & $\cfrac{2\alpha_m+\epsilon_m}{3}$ & $\cfrac{\alpha_m}{2}$ & $\cfrac{7\alpha_m+\epsilon_m}{8}$ & $\cfrac{2\epsilon_m}{7}$ & $\cfrac{\alpha_m}{2}$ & $\cfrac{2\alpha_m+\epsilon_m}{\sqrt{6}}$ \\
$(1, 1, 0, 0, 0)$ & $\sqrt{2}$ & $\left(1, \cfrac{2}{3}, 0, -\cfrac{1}{3},   -\cfrac{1}{3}\right)$ & $\sqrt{\cfrac{5}{3}}$ & $\cfrac{4\alpha_m+5\epsilon_m}{3}$ & $\epsilon_m$ & $\cfrac{19\alpha_m+5\epsilon_m}{8}$ & $2\alpha_m+\cfrac{\epsilon_m}{7}$ & $\cfrac{\epsilon_m}{\sqrt{2}}$ & $\cfrac{4\alpha_m+5\epsilon_m}{\sqrt{15}}$ \\
$(0, 1, 1, 0, 0)$ & $\sqrt{2}$ & $\left(0, \cfrac{2}{3}, 0, -\cfrac{1}{3},   -\cfrac{1}{3}\right)$ & $\sqrt{\cfrac{2}{3}}$ & $\cfrac{\alpha_m+2\epsilon_m}{3}$ & $\alpha_m$ & $\cfrac{3\alpha_m+\epsilon_m}{4}$ & $\cfrac{\epsilon_m}{7}$ & $\cfrac{\alpha_m}{\sqrt{2}}$ & $\cfrac{\alpha_m+2\epsilon_m}{\sqrt{6}}$ \\
$(0, 0, 1, 1, 0)$ & $\sqrt{2}$ & $\left(0, -\cfrac{1}{3}, 0, \cfrac{2}{3},   -\cfrac{1}{3}\right)$ & $\sqrt{\cfrac{2}{3}}$ & $\cfrac{\alpha_m-\epsilon_m}{3}$ & $\cfrac{\alpha_m}{2}$ & {\color[HTML]{FE0000} \textbf{$\cfrac{5\alpha_m+\epsilon_m}{2}$}} & $2\alpha_m+\cfrac{\epsilon_m}{7}$ & $\cfrac{\alpha_m}{2\sqrt{2}}$ & $\cfrac{\alpha_m-\epsilon_m}{\sqrt{6}}$ \\
$(0, 0, 0, 1, 1)$ & $\sqrt{2}$ & $\left(0, -\cfrac{2}{3}, 0, \cfrac{1}{3}, \cfrac{1}{3}\right)$ & $\sqrt{\cfrac{2}{3}}$ & $\cfrac{\alpha_m+2\epsilon_m}{3}$ & $0$ & $\cfrac{3\alpha_m+\epsilon_m}{4}$ & $0$ & $0$ & $\cfrac{\alpha_m+2\epsilon_m}{\sqrt{6}}$ \\
$(1, 1, 1, 0, 0)$ & $\sqrt{3}$ & $\left(1, \cfrac{2}{3}, 0, -\cfrac{1}{3},   -\cfrac{1}{3}\right)$ & $\sqrt{\cfrac{5}{3}}$ & $\cfrac{4\alpha_m+5\epsilon_m}{3}$ & $\epsilon_m$ & $0$ & $\cfrac{2\epsilon_m}{7}$ & $\cfrac{\epsilon_m}{sqrt{3}}$ & $\cfrac{4\alpha_m+5\epsilon_m}{\sqrt{15}}$ \\
$(0, 1, 1, 1, 0)$ & $\sqrt{3}$ & $\left(0, \cfrac{1}{3}, 0, \cfrac{1}{3}, -\cfrac{2}{3}\right)$ & $\sqrt{\cfrac{2}{3}}$ & $\cfrac{2\alpha_m+\epsilon_m}{3}$ & $\cfrac{\alpha_m}{2}$ & $\cfrac{7\alpha_m+\epsilon_m}{8}$ & $\cfrac{3\epsilon_m}{7}$ & $\cfrac{\alpha_m}{2\sqrt{3}}$ & $\cfrac{2\alpha_m+\epsilon_m}{\sqrt{6}}$ \\
$(0, 0, 1, 1, 1)$ & $\sqrt{3}$ & $\left(0, -\cfrac{2}{3}, 0, \cfrac{1}{3}, \cfrac{1}{3}\right)$ & $\sqrt{\cfrac{2}{3}}$ & $\cfrac{\alpha_m+2\epsilon_m}{3}$ & $0$ & $\cfrac{13\alpha_m+3\epsilon_m}{8}$ & $2\alpha_m+\cfrac{3\epsilon_m}{7}$ & $0$ & $\cfrac{\alpha_m+2\epsilon_m}{\sqrt{6}}$ \\
$(1, 1, 1, 1, 0)$ & $2$ & $\left(1, \cfrac{1}{3}, 0, \cfrac{1}{3}, -\cfrac{2}{3}\right)$ & $\sqrt{\cfrac{5}{3}}$ & $\cfrac{5\alpha_m+4\epsilon_m}{3}$ & $\cfrac{\alpha_m}{2}+\epsilon_m$ & $\cfrac{\alpha_m-\epsilon_m}{8}$ & $0$ & $\cfrac{\alpha_m}{4}+\cfrac{\epsilon_m}{2}$ & {\color[HTML]{FE0000} \textbf{$\cfrac{5\alpha_m+4\epsilon_m}{\sqrt{15}}$}} \\
$(0, 1, 1, 1, 1)$ & $2$ & $(0, 0, 0, 0, 0)$ & $0$ & $0$ & $\alpha_m$ & $0$ & $\cfrac{\epsilon}{7}$ & $\cfrac{\alpha_m}{2}$ & $0$ \\
$(1, 1, 1, 1, 1)$ & $\sqrt{5}$ & $(1, 0, 0, 0, 0)$ & $1$ & $\alpha_m+\epsilon_m$ & $\epsilon_m$ & $\cfrac{3\alpha_m+\epsilon_m}{4}$ & $\cfrac{2\epsilon}{7}$ & $\cfrac{\epsilon_m}{\sqrt{5}}$ & $\alpha_m+\epsilon_m$
\\
\hline
\end{tabular}
\end{sc}
\end{footnotesize}
\end{center}
\vskip -0.1in
\end{table}

\advantageprop*
\begin{proof}
We define $\X$ as $(1,2,3,4,5)$
and $\y^{(m)}=(-\alpha_m-\epsilon_m, \alpha_m, -3\alpha_m-\epsilon_m, \alpha_m+\epsilon_m, 2\alpha_m+\epsilon_m)$,
where $\alpha_m, \epsilon_m \in \R$ are two arbitrary sequences with $\alpha_m\rightarrow\infty$ and $\epsilon_m\rightarrow0$.
We calculate the values of $\gsobj$,$\gbobj$ and $\ogbobj$ for all possible queries to select the first, second and the third queries, as shown in Table \ref{tb:prop2proof1} and Table \ref{tb:prop2proof2}. 
We use the query vectors to represent the queries in this proof.

For the gradient sum objective, according to Table \ref{tb:prop2proof1}, the first query identified is either the one with outputs $\q_\gs^{(1)}=(1,1,1,0,0)$ or the one with output $\q_\gs^{(1)'}=(0,0,0,1,1)$ for the five data points. 
In the first case,
the weight of the query is $\weight_\gs^{(1)}=(-\alpha_m-2\epsilon_m/3)$.
The gradient vector after adding this rule is $\g_{\gs}^{(1)}=(-\epsilon_m/3,2\alpha_m+2\epsilon_m/3,-2\alpha_m-\epsilon_m/3, \alpha_m+\epsilon_m, 2\alpha_m+\epsilon_m)$.
The second query selected is $\q_\gs^{(2)}=(0,1,1,1,1)$ according to the objective values calculated in Table \ref{tb:prop2proof1}.
After adding this query, the corrected weight vector is 
$\weight_\gs^{(2)}=(-7\alpha_m/4-5\epsilon_m/4, 9\alpha_m/8+7\epsilon_m/8)$,
and the gradient vector is $\g_{\gs}^{(2)}=((3\alpha_m+\epsilon_m)/4, (13\alpha_m+3\epsilon_m)/8, -(19\alpha_m+5\epsilon_m)/8, -(\alpha_m-\epsilon_m)/8, (7\alpha_m+\epsilon_m)/8)$.
Then, we calculate the values of $\gsobj(\q)$ in Table \ref{tb:prop2proof2}, and the third query selected is $\q_\gs^{(3)}=(0,0,1,1,0)$.
The corrected weight vector is $\weight_\gs^{(3)}=(-(7\alpha_m+5\epsilon_m)/4, (19\alpha_m+9\epsilon_m)/8, -(5\alpha_m+\epsilon_m)/2)$.
The output vector of the rule ensemble is
$$-\frac{7\alpha_m+5\epsilon_m}{4}\begin{pmatrix}
1\\1\\1\\0\\0
\end{pmatrix}+\frac{19\alpha_m+9\epsilon_m}{8}\begin{pmatrix}
0\\1\\1\\1\\1
\end{pmatrix}-\frac{5\alpha_m+\epsilon_m}{2}\begin{pmatrix}
0\\0\\1\\1\\0
\end{pmatrix}=\begin{pmatrix}
-(7\alpha_m+5\epsilon_m)/4\\(5\alpha_m-\epsilon_m)/8\\-5(3\alpha_m+\epsilon_m)/8\\-(\alpha_m-5\epsilon_m)/8\\(19\alpha_m+9\epsilon)/8
\end{pmatrix}.$$
The gradient vector is $\g_{\gs}^{(3)}=((3\alpha_m+\epsilon_m)/4, (3\alpha_m+\epsilon_m)/8, -(9\alpha_m+3\epsilon_m)/8, (9\alpha_m+3\epsilon_m)/8, -(3\alpha_m+\epsilon_m)/8)$.
The empirical risk after adding three rules into the rule ensemble is 
$$R\left(f_\gs\left(\X, \y^{(m)}\right)\right)=\frac{3}{8}(3\alpha_m+\epsilon_m)^2.$$
In the second case ( $\q_\gs^{(1)'}=(0,0,0,1,1)$), the weight of the first query is $3\alpha_m/2+\epsilon_m$, and the gradient vector is $\g_{\gs}^{(1)'}=(-\alpha_m-\epsilon_m, \alpha_m, -3\alpha_m-\epsilon_m, -\alpha_m/2, \alpha_m/2)$.
The second query selected is $\q_\gs^{(2)'}=(1,1,1,1,0)$.
The corrected weight vector is $(2\alpha_m+9\epsilon_m,-\alpha_m-4\epsilon_m/7)$.
The gradient vector after adding two rules is $\g_{\gs}^{(2)'}=(-3\epsilon_m/7, 2\alpha_m+4\epsilon_m/7, -2\alpha_m-3\epsilon/7, 2\epsilon/7, -2\epsilon/7)$.
The third query selected is $\q_\gs^{(3)'}=(0,1,0,0,0)$.
The corrected weight vector is $((12\alpha_m+7\epsilon_m)/5, -(9\alpha_m+4\epsilon_m)/5, 2(7\alpha_m+2\epsilon_m)/5)$.
The output vector of the rule ensemble is
$$\frac{12\alpha_m+7\epsilon_m}{5}\begin{pmatrix}
0\\0\\0\\1\\1
\end{pmatrix}-\frac{9\alpha_m+4\epsilon_m}{5}\begin{pmatrix}
1\\1\\1\\1\\0
\end{pmatrix}+\frac{2}{5}(7\alpha_m+2\epsilon_m)\begin{pmatrix}
0\\1\\0\\0\\0
\end{pmatrix}=\begin{pmatrix}
-(9\alpha_m+4\epsilon_m)/5\\
\alpha_m\\
-(9\alpha_m+4\epsilon_m)/5\\
3(\alpha_m+\epsilon_m)/5\\
(12\alpha_m+7\epsilon_m)/5
\end{pmatrix}.$$
The empirical risk after adding three rules is 
$$R\left(f_\gs\left(\X, \y^{(m)}\right)\right)=\frac{2}{5}\left(6\alpha_m^2+2\alpha_m\epsilon_m+\epsilon_m^2\right).$$

\begin{figure*}[t]
\vskip 0.1in
\begin{center}
\centering
\includegraphics[width=0.95\columnwidth]{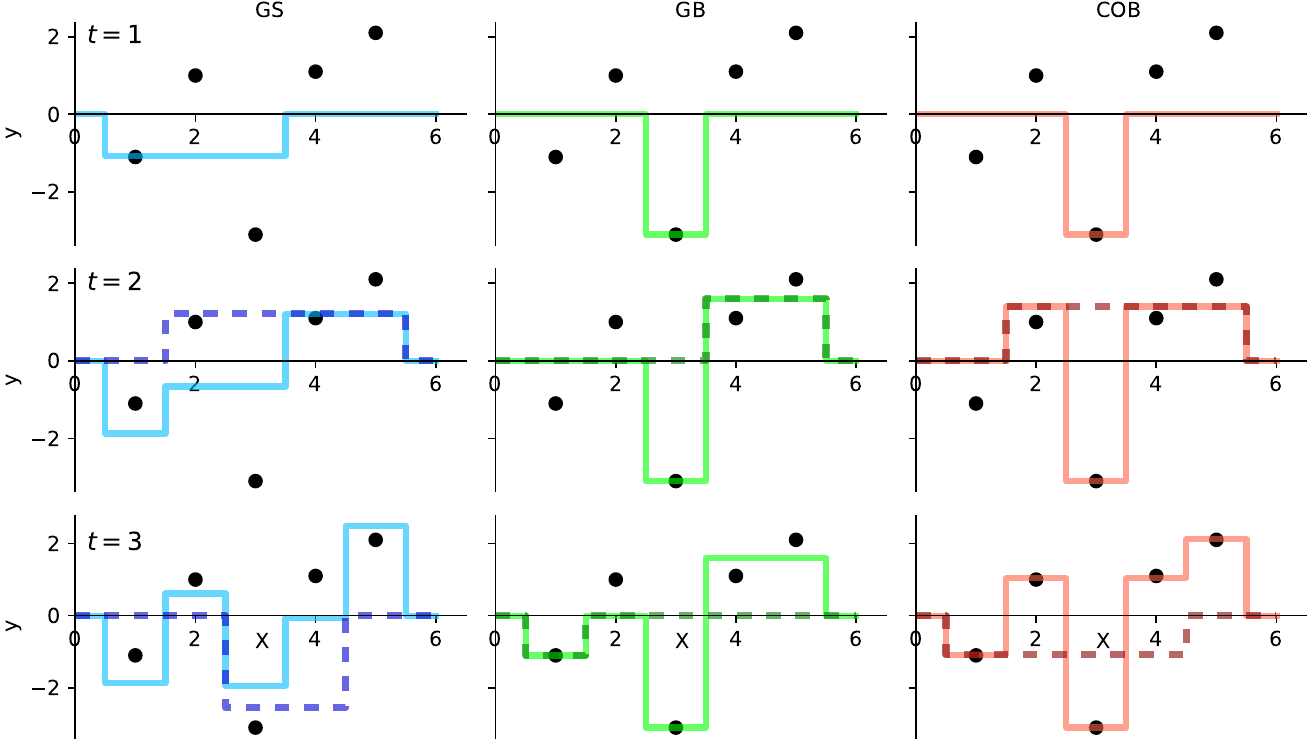}

\vskip -0.1in
\caption{Visualisation of the generating of rules by different objective functions (in column) in each iteration (row) for the dataset used in proof of Proposition \ref{prop:advantage}. The solid lines show the output of rule ensembles in each iteration. The dashed lines show the output of the rule generated in each iteration.}
\label{fig:vis_proof}
\end{center}
\vskip 0.1in
\end{figure*}

For gradient boosting objective, the first query selected is $\q_\gb^{(1)}=(0,0,1,0,0)$ according to Table \ref{tb:prop2proof1}. 
Its weight is $-3\alpha_m-\epsilon_m.$ 
The gradient after adding this rule is $\g_{\gb}^{(1)}=(-\alpha_m-\epsilon_m, \alpha_m, 0, \alpha_m+\epsilon_m, 2\alpha_m+\epsilon_m).$
The second query selected by the gradient boosting objective is $\q_\gb^{(2)}=(0,0,0,1,1).$ 
The weight vector after correction is $(-3\alpha_m-\epsilon_m, 3\alpha_m/2+\epsilon_m).$ 
The gradient becomes $\g_{\gb}^{(2)}=(-\alpha_m-\epsilon_m, \alpha_m, 0, -\alpha_m/2, \alpha_m/2)$ after adding the second rule.
Then, according to Table \ref{tb:prop2proof2}, the third query selected by gradient boosting objective is $\q_\gb^{(3)}=(1,0,0,0,0).$ 
The corrected weight vector is $(-3\alpha_m-\epsilon_m, 3\alpha_m/2+\epsilon_m, -\alpha_m-\epsilon_m).$
The output vector of the rule ensemble is
$$-(3\alpha_m+\epsilon_m)\begin{pmatrix}
0\\0\\1\\0\\0
\end{pmatrix}+\left(\frac{3\alpha_m}{2}+\epsilon_m\right)\begin{pmatrix}
0\\0\\0\\1\\1
\end{pmatrix}-(\alpha_m+\epsilon_m)\begin{pmatrix}
1\\0\\0\\0\\0
\end{pmatrix}=\begin{pmatrix}
-\alpha_m-\epsilon_m\\0\\-3\alpha_m-\epsilon_m\\3\alpha_m/2+\epsilon_m\\3\alpha_m/2+\epsilon_m
\end{pmatrix}.$$
The gradient vector is $\g_{\gb}^{(3)}=(0, \alpha_m, 0, -\alpha_m/2, \alpha_m/2).$
The empirical risk after adding three rules into the ensemble is 
$$
R\left(f_\gb\left(\X, \y^{(m)}\right)\right)=3\alpha_m^2/2.
$$

In the first iteration, the orthogonal boosting objective selects the same query as the gradient boosting, since their objective values are the same, so their weights, outputs, gradients are also the same.
We calculate the orthogonal projections $\q_\bot$, their lengths $\|\q_\bot\|$ and the objective values $\ogbobj(\q)$ as Table \ref{tb:prop2proof1}. 
The second query selected is $\q_\ogb^{(2)}=(0,1,1,1,1),$
and the weight vector after correction is $(-(13\alpha_m+5\epsilon_m)/3, (4\alpha_m+2\epsilon_m)/3).$
The gradient vector after adding the first two rules is $\g_{\ogb}^{(2)}=(-(\alpha_m+\epsilon_m),-(\alpha_m+\epsilon_m)/3,0,-(\alpha_m-\epsilon_m)/3,(2\alpha_m+\epsilon_m)/3).$
According to Table \ref{tb:prop2proof2}, the third query selected by the orthogonal boosting objective is $\q_\ogb^{(3)}=(1,1,1,1,0)$.
The weight vector is now $(-4\alpha_m-7\epsilon_m/5, 2\alpha_m+6\epsilon_m/5, -\alpha_m-4\epsilon_m/5).$
The output vector of the rule ensemble is
$$-\left(4\alpha_m+\frac{7\epsilon_m}{5}\right)\begin{pmatrix}
0\\0\\1\\0\\0
\end{pmatrix}+\left(2\alpha_m+\frac{6\epsilon_m}{5}\right)\begin{pmatrix}
0\\1\\1\\1\\1
\end{pmatrix}-\left(\alpha_m+\frac{4\epsilon_m}{5}\right)\begin{pmatrix}
1\\1\\1\\1\\0
\end{pmatrix}=\begin{pmatrix}
-\alpha_m-4\epsilon_m/5\\\alpha_m+2\epsilon_m/5\\-3\alpha_m-\epsilon_m\\\alpha_m+2\epsilon_m/5\\2\alpha_m+6\epsilon_m/5
\end{pmatrix}.$$
The gradient vector is $\g_{\ogb}^{(3)}=(-\epsilon_m/5, -2\epsilon_m/5, 0, 3\epsilon_m/5, -\epsilon_m/5).$
The empirical risk after adding three rules into the ensemble is 
$$
R\left(f_\ogb\left(\X, \y^{(m)}\right)\right)=3\epsilon_m^2/5.
$$

Since $\alpha_m\rightarrow\infty$ and $\epsilon_m\rightarrow0$, 
$$
\lim_{m\rightarrow\infty}{R\left(f_\gs\left(\X, \y^{(m)}\right)\right)}=\lim_{m\rightarrow\infty}{\frac{3}{8}\left(3\alpha_m+\epsilon_m\right)^2}=\infty,
$$
or 
$$
\lim_{m\rightarrow\infty}{R\left(f_\gs(\X, \y^{(m)})\right)}=\lim_{m\rightarrow\infty}{\frac{2}{5}\left(6\alpha_m^2+2\alpha_m\epsilon_m+\epsilon_m^2\right)}=\infty,
$$
$$
\lim_{m\rightarrow\infty}{R\left(f_\gb\left(\X, \y^{(m)}\right)\right)}=\lim_{m\rightarrow\infty}{3\alpha_m^2/2}=\infty,
$$
and 
$$
\lim_{m\rightarrow\infty}{R\left(f_\ogb\left(\X, \y^{(m)}\right)\right)}=\lim_{m\rightarrow\infty}{3\epsilon_m^2/5}=0.
$$

The rules generated by each objective function in each step is visualised as Figure \ref{fig:vis_proof}.

\end{proof}

\efficientprop*
\begin{proof}
    To see the claim, we first rewrite the objective value for the $i$-th prefix as
    \begin{equation*}
        \frac{\left|\inner{\g_\bot}{\q^{(i)}}\right|}{\|\q^{(i)}_\perp\| + \epsilon} = \frac{\left|\inner{\g_\bot}{\q^{(i)}}\right|}{\sqrt{\|\q^{(i)}\|^2 - \|\q^{(i)}_\parallel\|^2} + \epsilon} 
        \enspace .
    \end{equation*}
    The value of $\|\q^{(i)}\|^2$ is trivially given as $|\text{I}(\q^{(i)})|=i$, and $\inner{\g}{\q^{i}}$ can be easily computed for all $i \in [l]$ in time $O(n)$ via cumulative summation.
    Finally we can reduce the problem of computing the (squared) norms of the $l$ projected prefixes to computing the $t$ (squared) norms of the prefixes on the subspaces given by the individual orthonormal basis vectors via
    \begin{equation*}
        \|\q^{(i)}_\parallel\|^2 = \left\|\sum_{k=1}^t \veco_k \veco_k^T\q^{(i)}\right\|^2 = \sum_{k=1}^t \|\veco_k \veco_k^T\q^{(i)}\|^2 \enspace .
    \end{equation*}
    Each of these $t$ sequences of (squared) norms can be computed in time $O(n)$ by rewriting
    \begin{align*}
        \|\veco_k\veco_k^T\q^{(i)}\| &= \left\|\veco_k\veco_k^T\left(\sum_{j=1}^i\e_{\sigma(j)}\right)\right\|\\
        &=\|\veco_k\| \left|\sum_{j=1}^i \inner{\veco_k}{\e_{\sigma(j)}}\right|\\
        &=\left|\sum_{j=1}^i o_{k, \sigma(j)}\right| 
    \end{align*}
    where the last equality shows how an $O(n)$-computation is achieved via cumulative summation of the $k$-th basis vector elements in the order given by $\sigma$.
\end{proof}

\begin{proposition}
Let $\g$ be the gradient vector after the application of the weight correction step~\eqref{eq:weight_correction} for selected queries $\q_1, \dots, \q_t$. If the regularisation parameter is 0, then $\g \perp \spanof\{\q_1, \dots, \q_t\}$.
\label{th:g_orth}
\end{proposition}
\begin{proof}
    After the weight correction step $\weights$ is a stationary point of $R(\Q(\cdot))$, i.e., we have for all $j \in [t]$
    \begin{equation*}
        0 = \frac{\partial R(\Q\weights)}{\partial \beta_j}
          = \sum_{i=1}^n \frac{\partial\, l(\inner{\tilde{\q}_i}{\weights}, y_i)}{\partial \beta_j} 
          = \sum_{i=1}^n q_{ij} \underbrace{\frac{\partial\, l(\inner{\tilde{\q}_i}{\weights}, y_i)}{\partial\, \inner{\tilde{\q}_i}{\weights}}}_{g_i} = \inner{\q_j}{\g} \enspace .
    \end{equation*}
\end{proof}
\begin{proposition}
Let $\Q = [\q_1, \dots, \q_{t-1}] \in \R^{n \times (t-1)}$ be the selected query matrix and $\g$ the corresponding gradient vector after full weight correction, and let us denote by $\q=\q_\perp+\q_\parallel$ the orthogonal decomposition of $\q$ with respect to $\range\, \Q$. 
Then we have for a maximizer $\q^*$ of the \defemph{orthogonal gradient boosting objective} 
$
    \ogbobj(q) = |\inner{\g_\perp}{\q}|/(\|\q_\perp\|+\epsilon)
$:
\begin{itemize}
    \item[a)] For $\epsilon \to 0$, $\spanof\{\q_1, \dots, \q_{t-1}, \q^*\}$ is the best approximation to $\spanof\{\q_1, \dots, \q_{t-1}, \g\}$.
    \item[b)] For $\epsilon \to \infty$, $\q^*$ maximizes $\gsobj$ and any maximizer of $\gsobj$ maximizes $\ogbobj$.
    \item[c)] For $\epsilon = 0$ and $\|\q_\perp\|>0$, the ratio $(\ogbobj(q)/\gbobj(q))^2$ is equal to $1+(\|\q_\parallel\|/\|\q_\perp\|)^2$.
    \item[d)] The objective value $\ogbobj(q)$ is upper bounded by $\|\g_\perp\|$.
\end{itemize}
    \label{thm:objective}
\end{proposition}

\begin{proof}
\begin{itemize}
\item[a)] If $\epsilon\rightarrow0$, then $\ogbobj(q)\rightarrow\cfrac{|g_\bot^Tq|}{\norm{q_\bot}}$.
If $\q^*$ is a maximizer of $\ogbobj$, 
then as shown in Lemma 4.1, $\q^*$ minimises the minimum distance from all 
$$\f\in\text{span}\{\q_1, \cdots, \q_{t-1}, \g\}$$
to the subspace of $$\text{span}\{\q_1, \cdots, \q_{t-1}, \q^*\}.$$ 
Therefore, the subspace spanned by $[\q_1, \cdots, \q_{t-1}, \q^*]$ is the best approximation to the subspace spanned by $[\q_1, \cdots, \q_{t-1}, \g]$.


    
\item[b)] Let $q_1$ and $q_2$ be any two queries and denote by $\ogbobj^{(\epsilon)}(q)$ the $\ogbobj$-value of $q$ for a specific $\epsilon$.
Then
\begin{align*}
    &\lim_{\epsilon \to \infty} \epsilon\left(\ogbobj^{(\epsilon)}(q_1) - \ogbobj^{(\epsilon)}(q_2)\right) \\
    = &\lim_{\epsilon \to \infty} \epsilon\left( \frac{|\inner{g_\bot}{\q_1}|}{\|\q_1^\perp\|+\epsilon} - \frac{|\inner{g_\bot}{\q_2}|}{\|\q_2^\perp\|+\epsilon}\right) \\
    = &\lim_{\epsilon \to \infty} \left( \frac{|\inner{g_\bot}{\q_1}|}{\|\q_1^\perp\|/\epsilon+1} - \frac{|\inner{g_\bot}{\q_2}|}{\|\q_2^\perp\|/\epsilon+1}\right) \\
    = &|\inner{g_\bot}{\q_1}|-|\inner{g_\bot}{\q_2}|\\
    = &\gsobj(q_1) - \gsobj(q_2)
\end{align*}
Thus for large enough $\epsilon$, the signs of $\ogbobj^{(\epsilon)}(q_1) - \ogbobj^{(\epsilon)}(q_2)$ and $\gsobj(q_1) - \gsobj(q_2)$ agree.
Therefore, a query $q$ is a $\gsobj$-maximizer, i.e., $\gsobj(q) \geq \gsobj(q')$ for all $q' \in \cQ$, if and only if $q$ is a $\ogbobj$-maximizer, i.e., $\ogbobj(q) \geq \ogbobj(q')$ for all $q' \in \cQ$.

    


    
    
\item[c)] If $\epsilon=0$ and $\norm{q_\bot}>0$, then 
\begin{align*}
    \left(\frac{\ogbobj(q)}{\gbobj(q)}\right)^2&=\cfrac{\cfrac{|\g_\bot^T\q|^2}{\norm{\q_\bot}^2}}{\cfrac{|\g_\bot^T\q|^2}{\norm{\q}^2}} =\frac{\norm{\q}^2}{\norm{\q_\bot}^2} \\
    &=\frac{\norm{\q_\parallel}^2+\norm{\q_\bot}^2}{\norm{\q_\bot}^2}\\
    &=1+\left(\frac{\norm{\q_\parallel}}{\norm{\q_\bot}}\right)^2
\end{align*}
    

    
\item[d)] If we divide the numerator and denominator of $\ogbobj(\q)$ with $\norm{\q_bot}$, then we can get
\begin{align*}
    \ogbobj(\q)&=\frac{|\g_\bot^T\q|}{\|q_\perp\|+\epsilon} \\
    &= \cfrac{\cfrac{|\g_\bot^T\q_\bot|}{\norm{\q_\bot}}}{1+\cfrac{\epsilon}{\norm{\q_\bot}}}\\
    \intertext{according to the Cauchy–Schwarz inequality,  $\cfrac{|\g_\bot^T\q|}{\norm{\q_\bot}}\leq\cfrac{\norm{\g_\bot}\norm{\q_\bot}}{\norm{\q_\bot}}=\norm{\g_\bot}$, so, }\\
    \ogbobj(\q)&\leq \frac{\norm{\g_\bot}}{1+\cfrac{\epsilon}{\norm{\q_\bot}}}\\
    \intertext{as $\norm{\q_\bot}$ is upper bounded by the number of data points $n$, }\\
     \ogbobj(\q) & \leq \frac{\norm{\g_\bot}}{1+\cfrac{\epsilon}{n}}\\
     \ogbobj(\q) & \leq \norm{\g_\bot}.
\end{align*}
\end{itemize}
\end{proof}


\section{Greedy approximation to bounding function}

\begin{figure*}[htb]
\vskip 0.1in
\begin{center}
\centering
\includegraphics[width=0.8\columnwidth]{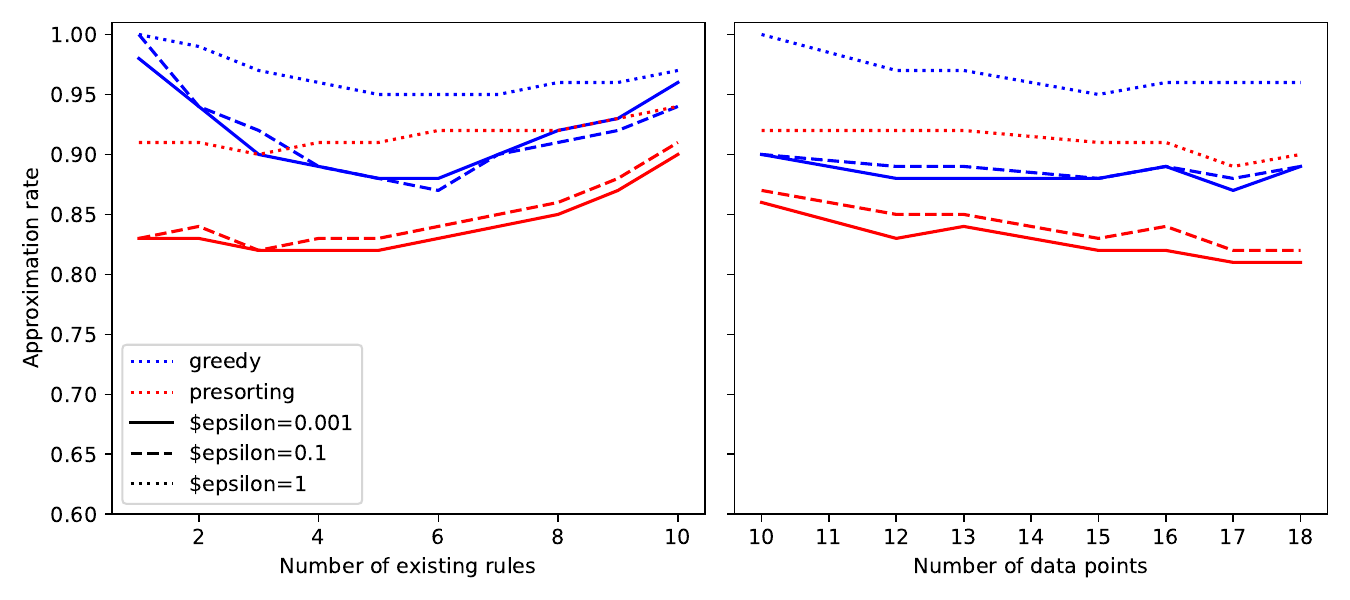}

\vskip -0.1in
\caption{The approximation rates for different number of existing rules (left) and data points (right) with $99\%$ success rate.}
\label{fig:greedy_approx}
\end{center}
\vskip 0.1in
\end{figure*}

The branch-and-bound search described in Section 3.3 requires an efficient way of calculating the value of $\text{bnd}(\q)=\max\{obj(\q'): I(\q')\subseteq I(\q), \q'\in\{0,1\}^n\}$, where $I(\q)=\{i: \q(x_i)=1,1\leq i\leq n\}$. 
It is too expensive to enumerate all possible $\q'$s as there are $2^n$ cases in the worst case. 
One solution to this problem is that we can relax the constraint $\q'\in\{0,1\}^n$ to $\q'\in[0,1]^n$ and it can be solved by quadratic programming. 
However, this would render the branch-and-bound search computationally very expensive. 
Instead, we investigate here using greedy approximations to identifying the optimal point sets. This does not yield an admissible bounding function, and thus does not guarantee to identify the optimal query, but can still lead to good approximation ratios in practice and is computationally inexpensive.

\begin{table}[t]
\begin{center}
\caption{The ratio of instances (15 data points and 5 existing rules) which reaches certain approximation rates.}
\begin{tabular}{rrrr}
\toprule
Approx. rate & $\epsilon$=0.001 & $\epsilon$=0.1 & $\epsilon$=1 \\
\midrule
75\% & 100.00\% & 100.00\% & 100.00\% \\
80\% & 99.66\% & 99.71\% & 100.00\% \\
85\% & 96.21\% & 98.05\% & 99.93\% \\
90\% & 88.11\% & 90.80\% & 99.34\% \\
95\% & 65.54\% & 70.15\% & 92.43\% \\
100\% & 33.20\% & 36.87\% & 63.28\%\\
\bottomrule
\end{tabular}
\end{center}
\end{table}

In particular, we are investigating two greedy variants: the fast prefix greedy approach described in the main text, and a slower full greedy approach that works as follows.
Given a query $\q'^{(t-1)}\leq\q$, we need to find the data point selected by $\q$ which maximise the objective function, and use it with $\q'^{(t-1)}$ to form a $\q'{(t)}$.
\begin{equation*}
    i_*^{(t)}=\argmax_{i\in I(q)-I(q'^{(t-1)})}{\frac{\g^T\left(\q'^{(t-1)}+\e_i\right)}{\norm{\left(\q'^{(t-1)}+\e_i\right)_\bot}+\epsilon}}. 
\end{equation*}
where 
$0\leq t\leq|I(\q)|$, $\q'^{(0)}=\0$ and $\q'^{(t)}=\q'^{(t-1)}+\e_{i_*^{(t)}}$. 
We use the maximum value of $\obj(\q'^{(t)})$ as the bounding value for query $\q$. 
The computation time complexity level of this approach is $O(n^2)$ for each query.


We now investigate the approximation ratios achieved by both approaches.
For that, we generate 2000 groups of initial queries and initial gradient vectors. 
Each of the groups contains 15 data points and 5 queries. In each query, each data point has a probability of 0.5 being 1 and 0. 
The initial gradient vector is originally generated by an $15$-dimensional standard normal distribution, and then it is projected onto the subspace orthogonal to the existing queries.
We test these 2000 instances to see the difference between the approximation of $\bnd(\q)$ obtained by the full greedy approach, the pre-sorting greedy approach, and the actual optimal objective values (obtained by a brute-force approach). We choose three different values of $\epsilon$: 0.001, 0.1 and 1.

Figure \ref{fig:greedy_approx} shows the approximation rate of different number of existing rules and number of data points. 
For the presorting greedy approach, if there are more rules existing, the approximation is closer to the true value.
Although the approximation rate is decreasing slightly with more number of data points, there is still a trend that the decreasing is getting smaller when the size of dataset is increasing.
The full greedy approach approximates the true bounding function better than the presorting greedy approach. 
However, since the fully greedy approach costs more time than the presorting greedy, it is still reasonable to use the presorting greedy approach to get higher efficiency.


\begin{figure*}[t]
\vskip 0.1in
\centering
\begin{center}
\includegraphics[width=0.55\columnwidth]{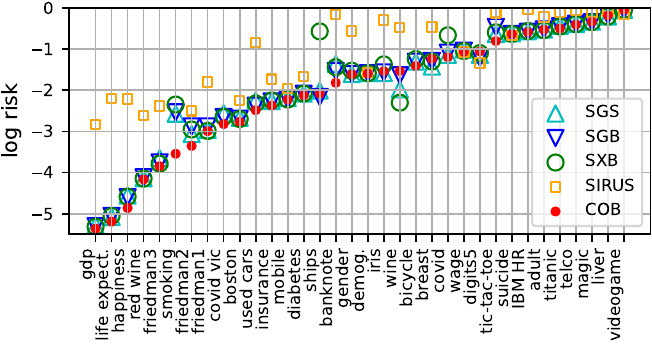}
\end{center}
\vskip -0.1in
\caption{Comparison of log training risks over different datasets. The datasets are ordered by the training risks of COB. 
}
\label{fig:log_train_risk}
\end{figure*}

\begin{figure*}[t]
\vskip 0.1in
\centering
\begin{center}
\includegraphics[width=0.41\columnwidth]{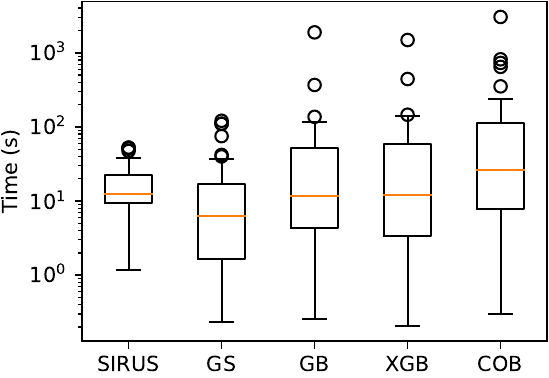}
\end{center}
\vskip -0.1in
\caption{The distribution of computation times across all test datasets to reach complexity level 50 for different algorithms. GS and GB are using greedy search, while XGB and COB are using branch-and-bound.
}
\label{fig:time_box}
\end{figure*}

\begin{table*}[t]
\caption{Comparison of normalised risks and computation times for rule ensembles, averaged over cognitive complexities between 1 and 50, using SIRUS(SRS), Gradient Sum(SGS), Gradient boosting (SGB), XGBoost (SXB) and COB (using greedy search and branch-and-bound search), for benchmark datasets of classification (upper), regression (middle) and Poisson regression problems (lower).}
\label{tb:comparison}
\vskip 0.05in
\begin{center}
\begin{scriptsize}
\begin{sc}
\begin{tabular}{l@{\hskip 0.03in}c@{\hskip 0.03in}c@{\hskip 0.03in}c@{\hskip 0.03in}c@{\hskip 0.03in}c@{\hskip 0.03in}c@{\hskip 0.03in}c@{\hskip 0.03in}c@{\hskip 0.03in}c@{\hskip 0.03in}c@{\hskip 0.03in}c@{\hskip 0.03in}c@{\hskip 0.03in}c@{\hskip 0.03in}c@{\hskip 0.03in}c@{\hskip 0.03in}c@{\hskip 0.03in}c@{\hskip 0.03in}c@{\hskip 0.03in}c}
\toprule
\multicolumn{1}{c}{\multirow{2}{*}{Dataset}} & \multicolumn{1}{c}{\multirow{2}{*}{$d$}} & \multicolumn{1}{c}{\multirow{2}{*}{$n$}} & \multicolumn{6}{c}{Train risks} & \multicolumn{6}{c}{Test risks} & \multicolumn{5}{c}{Computation times} \\
\multicolumn{1}{c}{} & \multicolumn{1}{c}{} & \multicolumn{1}{c}{} & SRS & SGS & SGB & SXB & COB$_{G}$ & COB$_{B}$ & SRS & SGS & SGB & SXB & COB$_{G}$ & COB$_{B}$ & SRS & SGS & SGB & SXB & COB$_{B}$ \\ \hline
titanic & 7 & 1043 & .895 & .653 & .656 & .646 & {\color[HTML]{FF0000} .639} & {\color[HTML]{FF0000} \textbf{.616}} & .894 & .695 & \textbf{.707} & .711 & .713 & .717 & 7.077 & 2.624 & 9.858 & 10.21 & 25.71 \\
tic-tac-toe & 27 & 958 & .892 & .555 & .641 & .577 & .650 & {\color[HTML]{FF0000} \textbf{.492}} & .885 & .596 & .682 & .629 & .721 & {\color[HTML]{FF0000} \textbf{.566}} & 12.59 & 3.971 & 10.34 & 6.09 & 13.99 \\
iris & 4 & 150 & .685 & .261 & .261 & .332 & {\color[HTML]{FF0000} \textbf{.192}} & {\color[HTML]{FF0000} .251} & .745 & .521 & .440 & .459 & .531 & {\color[HTML]{FF0000} \textbf{.424}} & 11.02 & 0.775 & 1.099 & 1.453 & 2.487 \\
breast & 30 & 569 & .569 & \textbf{.277} & .310 & .314 & .290 & .304 & .627 & \textbf{.269} & .362 & .338 & .383 & .349 & 11.48 & 6.744 & 74.43 & 74.83 & 239.2 \\
wine & 13 & 178 & .578 & .216 & .250 & .192 & {\color[HTML]{FF0000} .191} & {\color[HTML]{FF0000} \textbf{.183}} & .621 & .346 & .431 & .409 & .483 & {\color[HTML]{FF0000} \textbf{.265}} & 9.456 & 1.530 & 4.432 & 2.154 & 55.183 \\
ibm hr & 32 & 1470 & .980 & .567 & .560 & .571 & {\color[HTML]{FF0000} \textbf{.558}} & {\color[HTML]{FF0000} \textbf{.558}} & .974 & .640 & .645 & .636 & .652 & {\color[HTML]{FF0000} \textbf{.621}} & 11.15 & 17.24 & 10.99 & 12.92 & 12.03 \\
telco churn & 18 & 7043 & .944 & .679 & .682 & .678 & {\color[HTML]{FF0000} \textbf{.664}} & {\color[HTML]{FF0000} .668} & .945 & .663 & .677 & .665 & {\color[HTML]{FF0000} \textbf{.650}} & {\color[HTML]{FF0000} .660} & 50.83 & 40.01 & 1883 & 1485 & 3039 \\
gender & 20 & 3168 & .566 & .230 & .230 & .249 & {\color[HTML]{FF0000} .224} & {\color[HTML]{FF0000} \textbf{.224}} & .570 & \textbf{.243} & .247 & .263 & .246 & .246 & 22.42 & 22.73 & 25.49 & 24.27 & 32.95 \\
banknote & 4 & 1372 & .854 & .304 & .267 & .290 & {\color[HTML]{FF0000} .253} & {\color[HTML]{FF0000} \textbf{.228}} & .858 & .311 & .268 & .299 & {\color[HTML]{FF0000} .264} & {\color[HTML]{FF0000} \textbf{.229}} & 8.933 & 6.298 & 5.648 & 7.060 & 8.444 \\
liver & 6 & 345 & .908 & .815 & .834 & .814 & {\color[HTML]{FF0000} \textbf{.802}} & .834 & .917 & .879 & .927 & \textbf{.873} & .940 & .891 & 9.734 & 1.997 & 99.72 & 124.1 & 193.9 \\
magic & 10 & 19020 & .906 & .718 & .708 & .710 & {\color[HTML]{FF0000} .707} & {\color[HTML]{FF0000} .707} & .903 & .698 & .693 & .693 & {\color[HTML]{FF0000} .688} & {\color[HTML]{FF0000} \textbf{.688}} & 1.364 & 75.14 & 89.18 & 101.9 & 352.2 \\
adult & 11 & 30162 & .804 & .594 & .599 & .588 & {\color[HTML]{FF0000} \textbf{.575}} & {\color[HTML]{FF0000} .576} & .802 & .603 & .615 & .601 & {\color[HTML]{FF0000} \textbf{.589}} & {\color[HTML]{FF0000} .589} & 2.169 & 121.0 & 136.7 & 146.0 & 728.3 \\
digits5 & 64 & 3915 & \textbf{.248} & .332 & .312 & .344 & .329 & .315 & \textbf{.262} & .329 & .314 & .341 & .320 & .315 & 52.60 & 110.8 & 72.74 & 101.5 & 97.4 \\\hline
insurance & 6 & 1338 & .169 & .130 & .142 & .144 & {\color[HTML]{FF0000} \textbf{.120}} & {\color[HTML]{FF0000} .123} & .177 & .132 & .145 & .147 & {\color[HTML]{FF0000} \textbf{.127}} & {\color[HTML]{FF0000} .128} & 14.06 & 7.507 & 15.94 & 12.98 & 39.53 \\
friedman1 & 10 & 2000 & .180 & .089 & .074 & .068 & {\color[HTML]{FF0000} .067} & {\color[HTML]{FF0000} .068} & .165 & .091 & .077 & .075 & {\color[HTML]{FF0000} \textbf{.070}} & {\color[HTML]{FF0000} .072} & 16.79 & 2.514 & 4.302 & 3.171 & 6.915 \\
friedman2 & 4 & 10000 & .082 & .133 & .119 & .115 & {\color[HTML]{FF0000} .770} & {\color[HTML]{FF0000} \textbf{.075}} & .082 & .135 & .120 & .115 & {\color[HTML]{FF0000} \textbf{.077}} & {\color[HTML]{FF0000} \textbf{.077}} & 47.33 & 11.79 & 17.56 & 13.18 & 28.4 \\
friedman3 & 4 & 5000 & .093 & .045 & .042 & .042 & {\color[HTML]{FF0000} .041} & {\color[HTML]{FF0000} \textbf{.041}} & .092 & .048 & .047 & .046 & {\color[HTML]{FF0000} .045} & {\color[HTML]{FF0000} \textbf{.045}} & 29.86 & 6.243 & 10.61 & 8.559 & 17.65 \\
wage & 5 & 1379 & .427 & .370 & .362 & .359 & {\color[HTML]{FF0000} .352} & {\color[HTML]{FF0000} .354} & \textbf{.341} & .358 & .405 & .411 & .368 & .365 & 14.18 & 5.605 & 12.12 & 13.17 & 25.19 \\
demographics & 13 & 6876 & .219 & .214 & .214 & .214 & {\color[HTML]{FF0000} .212} & {\color[HTML]{FF0000} \textbf{.212}} & \textbf{.209} & .216 & .217 & .217 & .214 & .215 & 38.24 & 36.80 & 29.40 & 33.04 & 72.42 \\
gdp & 1 & 35 & .063 & .020 & .020 & .020 & .024 & {\color[HTML]{FF0000} \textbf{.020}} & .059 & .020 & .020 & .020 & .027 & {\color[HTML]{FF0000} \textbf{.020}} & 7.974 & .261 & .351 & .282 & .488 \\
used cars & 4 & 1770 & .373 & .139 & .123 & .132 & {\color[HTML]{FF0000} \textbf{.113}} & {\color[HTML]{FF0000} \textbf{.121}} & .427 & .171 & .131 & .141 & {\color[HTML]{FF0000} \textbf{.116}} & {\color[HTML]{FF0000} .130} & 15.00 & 8.371 & 12.10 & 9.484 & 20.27 \\
diabetes & 10 & 442 & .156 & .138 & .142 & .139 & {\color[HTML]{FF0000} .132} & {\color[HTML]{FF0000} .134} & .188 & .141 & .141 & .147 & {\color[HTML]{FF0000} \textbf{.136}} & .149 & 10.50 & 2.204 & 3.574 & 3.920 & 7.591 \\
boston & 13 & 506 & .101 & .086 & .087 & .086 & {\color[HTML]{FF0000} \textbf{.080}} & {\color[HTML]{FF0000} .082} & .105 & \textbf{.079} & .087 & .089 & .088 & .087 & 10.96 & 3.055 & 6.731 & 5.285 & 10.44 \\
happiness & 8 & 315 & .109 & .031 & .031 & .031 & {\color[HTML]{FF0000} .029} & {\color[HTML]{FF0000} \textbf{.029}} & .109 & .033 & .039 & .039 & .035 & {\color[HTML]{FF0000} \textbf{.033}} & 6.344 & 1.160 & 11.37 & 11.31 & 26.43 \\
life expect. & 21 & 1649 & .109 & .026 & .026 & .026 & {\color[HTML]{FF0000} .026} & {\color[HTML]{FF0000} \textbf{.026}} & .110 & .027 & .027 & .027 & {\color[HTML]{FF0000} .026} & {\color[HTML]{FF0000} \textbf{.026}} & 21.44 & 16.16 & 58.43 & 63.82 & 131.2 \\
mobile prices & 20 & 2000 & .148 & .131 & .137 & .137 & {\color[HTML]{FF0000} \textbf{.122}} & {\color[HTML]{FF0000} .126} & .140 & .134 & .143 & .143 & {\color[HTML]{FF0000} \textbf{.126}} & {\color[HTML]{FF0000} .132} & 33.81 & 15.03 & 367.7 & 442.5 & 815.4 \\
suicide rate & 5 & 27820 & .547 & .543 & .540 & .540 & {\color[HTML]{FF0000} .540} & {\color[HTML]{FF0000} .532} & .514 & .521 & .521 & .521 & {\color[HTML]{FF0000} .519} & {\color[HTML]{FF0000} .512} & 52.35 & 109.6 & 117.1 & 139.6 & 644.6 \\
videogame & 6 & 16327 & \textbf{.895} & .953 & .953 & .953 & .953 & .953 & .850 & \textbf{.720} & \textbf{.720} & \textbf{.720} & {\color[HTML]{FF0000} \textbf{.720}} & {\color[HTML]{FF0000} \textbf{.720}} & 1.171 & 41.91 & 34.38 & 45.90 & 119.1 \\
red wine & 11 & 1599 & .072 & .034 & .035 & .034 & {\color[HTML]{FF0000} .034} & {\color[HTML]{FF0000} .034} & .073 & .035 & .036 & .036 & {\color[HTML]{FF0000} .035} & {\color[HTML]{FF0000} \textbf{.035}} & 19.94 & 9.149 & 15.32 & 21.99 & 35.34 \\\hline
covid vic & 4 & 85 & NA & .153 & .121 & .132 & {\color[HTML]{FF0000} .105} & {\color[HTML]{FF0000} \textbf{.086}} & NA & .182 & \textbf{.100} & .133 & .104 & .086 & NA & .523 & .600 & .628 & .854 \\
covid & 2 & 225 & NA & .344 & .371 & .891 & {\color[HTML]{FF0000} .343} & {\color[HTML]{FF0000} \textbf{.321}} & NA & .459 & .411 & .741 & .417 & {\color[HTML]{FF0000} \textbf{.395}} & NA & .701 & .690 & .682 & 1.143 \\
bicycle & 4 & 122 & NA & .317 & .324 & .337 & {\color[HTML]{FF0000} .296} & {\color[HTML]{FF0000} \textbf{.275}} & NA & .366 & .478 & .457 & .529 & {\color[HTML]{FF0000} \textbf{.320}} & NA & .695 & 1.103 & 1.105 & 2.124 \\
ships & 4 & 34 & NA & .174 & .181 & .146 & {\color[HTML]{FF0000} \textbf{.125}} & .168 & NA & .197 & .203 & .199 & .464 & {\color[HTML]{FF0000} \textbf{.155}} & NA & .235 & .296 & .311 & .448 \\
smoking & 2 & 36 & NA & .127 & .128 & .163 & {\color[HTML]{FF0000} \textbf{.078}} & {\color[HTML]{FF0000} .072} & NA & .136 & .250 & .322 & {\color[HTML]{FF0000} .121} & {\color[HTML]{FF0000} \textbf{.084}} & NA & .266 & .256 & .208 & .301\\
\bottomrule
\end{tabular}
\end{sc}
\end{scriptsize}
\end{center}
\vskip -0.2in
\end{table*}

\section{Additional Details of Empirical Evaluation}
\label{sec:SI:eval}
All experiments in this paper are conducted on a computer with CPU ‘Intel(R) Core(TM) i5-10300H CPU @ 2.50GHz’ and memory of 24G.

To further show the difference between the proposed corrective orthogonal boosting and the other methods, we provide additional details of empirical evaluation.
Table \ref{tb:comparison} shows the normalised average training risk, test risk and computation time of the rule ensembles generated by SIRUS, SGS, SGB, SXB and COB using greedy search and  branch-and-bound search over complexity levels from 1 to 50 for the 34 datasets used in the experiments of this paper.
In Table \ref{tb:comparison}, we bold the lowest training and test risks for each dataset, and the texts with red colours indicate the COB approach using greedy search or branch-and-bound search have lower risks than all the other methods.
Figure \ref{fig:log_train_risk} compares the normalised average logged training risks over complexity levels from 1 to 50 for different datasets. 
According to Fig. \ref{fig:log_train_risk} and Table \ref{tb:comparison}, COB generates lower training risks than the other algorithms for 26 out of 34 datasets. 

For the COB with greedy search, there are 29 out of 34 datasets whose training risks are lower than the other methods. However, it has only 15 out of 34 datasets whose test risks are lower than other methods. Therefore, using branch-and-bound search generates better rule ensembles than greedy search. 
The One-sided T-test at significance level 0.05 with Bonferroni-correction for 8 hypotheses (4 for training and 4 for test) also shows the same results with the branch-and-bound search: the COB using greedy search generates rule ensembles with significantly less risk than the other methods by at least 0.001 of the normalised training and test risks.

Figure \ref{fig:time_box} shows the box plot of the running time for generating rule ensembles with complexity level 50 spent by different algorithms on the 34 datasets. Although the overall running time of COB is higher than the other methods, they are still at the same scale.

Furthermore, Figure \ref{fig:compare_all} shows more comparisons of the risk / complexity tradeoff for SIRUS, gradient sum, gradient boosting, XGBoost and Orthogonal Gradient Boosting for 6 datasets. 
We compare the rule ensembles generated by COB with complexity around 20 and the first method whose risk value is compatitive with the COB rule ensemble.

\begin{figure}[t!]
\vskip 0.1in
\begin{center}
\begin{subfigure}
\centering
\includegraphics[width=0.95\columnwidth]{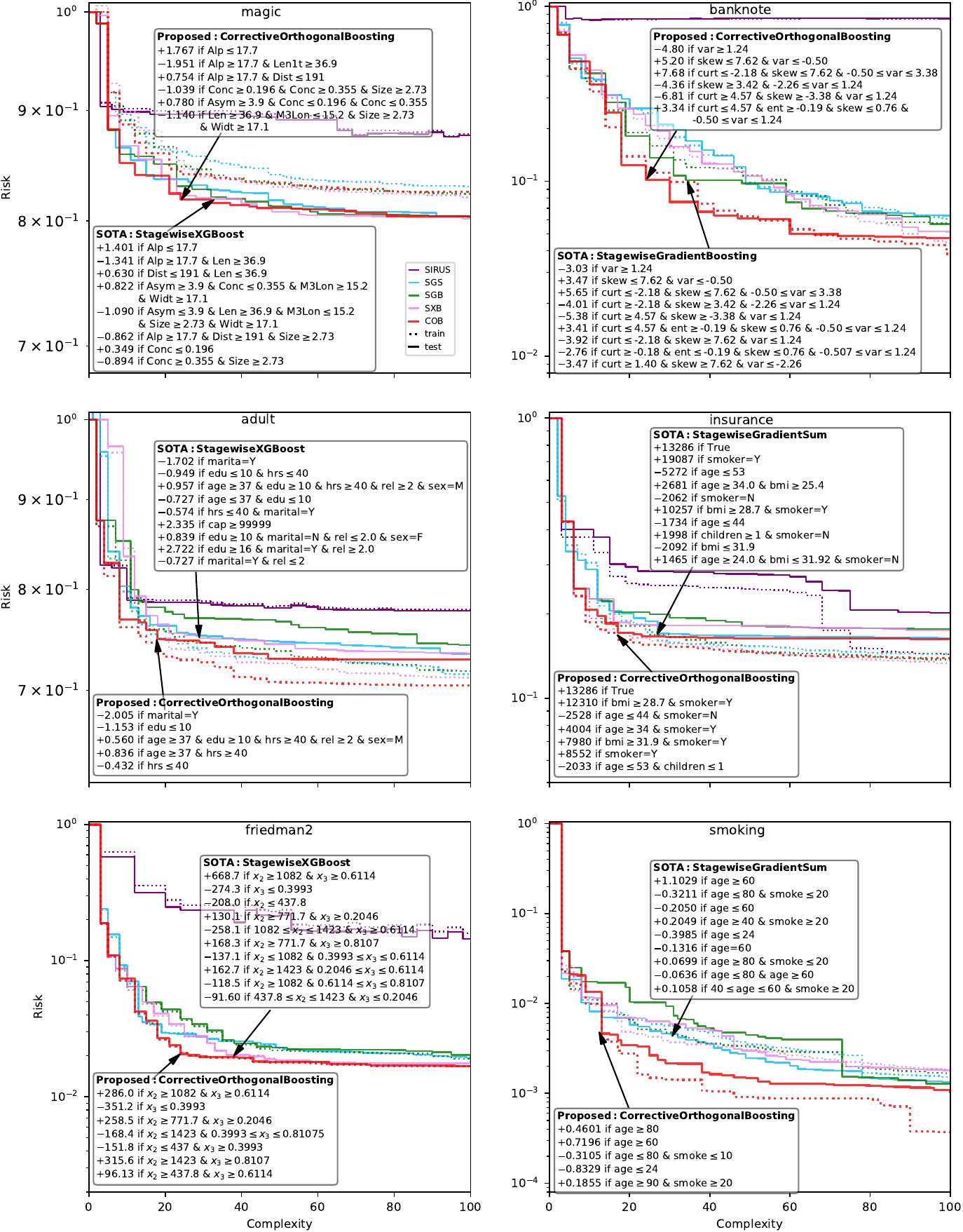}
\vspace{-0.5cm}
\caption{Risk/complexity curves of proposed  approach (red) compared to alternatives for  \texttt{magic}, \texttt{banknotes}, \texttt{adult}, \texttt{insurance}, \texttt{friedman 2} and \texttt{smoking}. 
Annotated rule ensembles have  equivalent risk but substantially reduced complexity for the proposed method.
}
\label{fig:compare_all}
\end{subfigure}
\end{center}
\end{figure}

\end{appendix}

\end{document}